\documentclass{article}
 \usepackage[final,main]{neurips_2026}
\usepackage[utf8]{inputenc} % allow utf-8 input
\usepackage[T1]{fontenc}    % use 8-bit T1 fonts
\usepackage{hyperref}       % hyperlinks
\usepackage{url}            % simple URL typesetting
\usepackage{booktabs}       % professional-quality tables
\usepackage{amsfonts}       % blackboard math symbols
\usepackage{nicefrac}       % compact symbols for \nicefrac{1}{2}, etc.
\usepackage{microtype}      % microtypography
\usepackage{xcolor}         % colors
\usepackage{wrapfig}
\usepackage{needspace}

\usepackage{times}
\usepackage{graphicx}
\usepackage{hyperref}
\usepackage{url}
\usepackage{algorithm}
\usepackage{algorithmic}
\usepackage{amsmath}
\usepackage{amsthm}
\usepackage{amssymb}
\usepackage[capitalize,noabbrev]{cleveref}
\usepackage{appendix}
\usepackage{apptools}
\usepackage{interval}
\usepackage{subcaption}  % 删除 \usepackage{subfigure}
\usepackage{fancybox}

\theoremstyle{definition} 
\newtheorem{lemma}{Lemma}

\newtheorem{theorem}{Theorem}

\newtheorem{assumption}{Assumption}
\usepackage{booktabs}
\usepackage{threeparttable}
\usepackage{xcolor}         % colors
\usepackage{colortbl}
\usepackage{makecell}
\usepackage{bbding}
\usepackage{parskip}
\definecolor{customgray}{gray}{0.9}
\usepackage{xcolor}

\hypersetup{
    colorlinks=true,
    linkcolor=blue,
    citecolor=blue,
    urlcolor=blue
}

\usepackage{bm}

\DeclareMathOperator{\op}{op}
\DeclareMathOperator{\tr}{Tr}

\DeclareMathOperator{\diag}{diag}

\DeclareMathOperator*{\sgn}{sign}
\DeclareMathOperator*{\msgn}{msign}

\newcommand{\R}{\mathbb{R}}

\newcommand{\brac}[1]{\left(#1\right)}
\newcommand{\mbrac}[1]{\left[#1\right]}
\newcommand{\sqbrac}[1]{\left[#1\right]}
\newcommand{\bbrac}[1]{\left \{ #1 \right\}}
\newcommand{\abs}[1]{\left|#1\right|}
\newcommand{\norm}[1]{\lVert#1\rVert}
\newcommand{\inner}[2]{\left\langle #1, #2 \right\rangle}

\newcommand{\Tr}[1]{\tr \left( #1 \right)}

\newcommand{\expect}[1]{\E\sqbrac{#1}}
\newcommand{\bigO}[1]{\mathcal{O}\left(#1\right)}
\newcommand\vsigma{{\bm{\sigma}}}
\newcommand\vSigma{{\bm{\Sigma}}}
\newcommand{\msign}[1]{\msgn\left(#1\right)}
\newcommand{\sign}[1]{\sgn\left(#1\right)}

\def \y {\mathbf{y}}
\def \YB {\mathbf{Y}}
\def \E {\mathbb{E}}
\def \x {\mathbf{x}}
\def \XB {\mathbf{X}}

\def \V {\mathbf{V}}
\def \g {\mathbf{g}}
\def \L {\mathbf{L}}

\def \U {\mathbf{U}}

\def \WB {\mathbf{W}}

\def \R {\mathbb{R}}

\def \PB {\mathbf{P}}
\def \F {\mathcal{F}}

\def \q {\mathbf{q}}

\def \q {\mathbf{q}}

\def \Ib {\mathbf{I}}

\def \SBB {\mathbb{S}}

\def \Gb {\mathbf{G}}

\def \Fn {\textnormal{F}}

\def \SB {\mathbf{S}}

\def \MB {\mathbf{M}}

\def \NB {\mathbf{N}}

\def \QB {\mathbf{Q}}

\def \e {\mathbf{e}}
\def \EB {\mathbf{E}}

\def \Pr    {\mathbb{P}}

\def \bSigma {\bm{\Sigma}}

\def \op {\textnormal{op}}

\def \SumD {\sum_{i=1}^d}

\def \Linf {L_{\infty}}
\def\0{\bm{0}}

\crefname{assumption}{Assumption}{Assumptions}
\crefname{lemma}{Lemma}{Lemmas}
\crefname{theorem}{Theorem}{Theorems}
\crefname{appendix}{Appendix}{Appendices}
\crefname{subtheorem}{Theorem}{Theorems}

\DeclareHookRule{begindocument}{parskip}{before}{hyperref}
\usepackage{tcolorbox}
\definecolor{comparisonblue}{HTML}{668FBC}
\definecolor{comparisongold}{HTML}{CF951F}
\newtcolorbox{comparisonpanel}[2][]{
    nobeforeafter,width=0.485\linewidth,
    colback=#2!10!white,colframe=#2,
    boxrule=0.5pt,arc=2mm,boxsep=0pt,
    left=4pt,right=4pt,top=2pt,bottom=2pt,
    halign=center,
    before upper={
    \setlength{\parskip}{0pt}%
        \setlength{\abovedisplayskip}{3pt}%
        \setlength{\abovedisplayshortskip}{3pt}%
        \setlength{\belowdisplayskip}{0pt}%
        \setlength{\belowdisplayshortskip}{0pt}
        },
    #1
}
\newenvironment{formulacomparison}
    {\par\addvspace{2pt}\setlength{\parskip}{0pt}%
        \nointerlineskip\noindent\ignorespaces}
    {\par\addvspace{2pt}\nointerlineskip\ignorespacesafterend}

\title{When and Why SignSGD Outperforms SGD: A Theoretical Study Based on $\ell_1$-norm Lower Bounds}

\author{
    Hongyi Tao\thanks{Equal Contribution. \textsuperscript{$\dagger$} Corresponding Author.} \quad
    Dingzhi Yu\footnotemark[1] \quad
    Lijun Zhang\textsuperscript{$\dagger$}\\
    State Key Laboratory of Novel Software Technology,\\
    Nanjing University, Nanjing 210023, China\\
    School of Artificial Intelligence, Nanjing University, Nanjing 210023, China\\
    \texttt{\{taohy,yudz,zhanglj\}@lamda.nju.edu.cn}
}

\begin{document}
\maketitle

\begin{abstract}
Sign-based optimization algorithms, such as SignSGD and Muon, have garnered significant attention for their remarkable performance in training large foundation models. Despite this empirical success, we still lack a theoretical understanding of when and why these sign-based methods outperform vanilla SGD. The core obstacle is that under standard smoothness and finite variance conditions, SGD is known to be minimax optimal for finding stationary points measured by $\ell_2$-norms, thereby fundamentally precluding any complexity gains for sign-based methods in standard settings. To overcome this barrier, we analyze sign-based optimizers leveraging $\ell_1$-norm stationarity, $\ell_\infty$-smoothness, and a separable noise model, which can better capture the coordinate-wise nature of signed updates. Under this distinct problem geometry, we derive matched upper and lower bounds for SignSGD and explicitly characterize the problem class in which SignSGD provably dominates SGD. Specifically, we compare the \emph{upper bound of SignSGD} with the \emph{lower bound of SGD}, illustrating that SignSGD effectively reduces the complexity by a factor of $d$ under \emph{sparse noise}, where $d$ is the problem dimension. Furthermore, we elevate this framework to the matrix domain, providing an equivalent optimal lower bound for the Muon optimizer, proving that extending the sign operator to matrices preserves this optimal scaling with dimensionality. Finally, we bridge our theoretical bounds to practice, demonstrating that the theoretical superiority of SignSGD accurately predicts its faster convergence during the pretraining of nanoGPT and LLaMA, with Muon also outperforming SGD in these experiments. Code is available at \url{https://github.com/Dingzhen230/SignSGD_Outperforms_SGD}.
\end{abstract}

\section{Introduction} 
Efficient and scalable stochastic optimization algorithms play an indispensable role in the huge success of large foundation models~\citep{devlin-etal-2019-bert, brown2020gpt3, achiam2023gpt4, touvron2023llama, team2023gemini}. Among these, sign-based optimization algorithms, such as SignSGD~\citep{bernstein2018signsgd} and Muon~\citep{jordan2024muon}, have attracted increasing focus due to substantial empirical edges over stochastic gradient descent (SGD)~\citep{robbins1951stochastic}. SignSGD leverages $1$-bit signed gradients to update the model, which naturally enjoys many favorable empirical properties in distributed environments~\citep{bernstein2019signsgd} and low-precision training regimes~\citep{yu2026stosignsgd}. Muon optimizes using the matrix sign operator, implemented via Newton--Schulz iterations~\citep{kovarik1970some,bjorck1971iterative}. Numerous empirical studies~\citep{liu2025muon,shah2025practical,wen2025fantastic,semenov2025benchmarking} have shown its consistent speedup over AdamW~\citep{loshchilov2019adamw}, and it has become a new industrial paradigm for pretraining massive-scale large language models (LLMs)~\citep{team2025kimi,kimi2026kimi2.5,zeng2025glm,zeng2026glm,cheng2026ngram,deepseekai2026deepseekv4}.

The empirical success of sign-based optimizers has motivated a fruitful line of work that attempts to justify their advantages theoretically. From an optimization theory perspective, the convergence of sign-based methods under various assumptions has been extensively studied~\citep{bernstein2018signsgd,bernstein2019signsgd,safaryan2021stochastic,sun2023momentum,jiang2025improved,yu2026signheavytails}. However, a complete theoretical understanding of why sign-based optimizers outperform SGD remains elusive. The biggest obstacle preventing any theoretical advances is that \emph{SGD is already worst-case optimal when the objective function is smooth and stochastic gradients are unbiased with bounded variance~\citep{arjevani2023lower}}. Under these standard assumptions, any first-order algorithm requires at least $\Omega(\epsilon^{-4})$ queries to find an $\epsilon$-stationary point measured by the $\ell_2$-norm, a complexity tightly matched by SGD~\citep{ghadimi2013stochastic}. Consequently, evaluating the performance of sign-based methods under the aforementioned problem class and $\ell_2$-stationary measure inherently fails to deliver any provable gains we desire. 

In this paper, we systematically investigate when and why sign-based methods can outperform SGD by delving into a new problem class and performance measure that depart from the canonical setting in~\citet{arjevani2023lower}. Specifically, instead of the traditional $\ell_2$-geometry based on $\ell_2$-smoothness, finite variance, and $\ell_2$-stationarity, we consider an $\ell_\infty$-geometry with $\ell_\infty$-smoothness, coordinate-wise finite variance, and $\ell_1$-stationarity. This geometry shift is motivated by the insight that sign-based updates are intrinsically aligned with $\ell_\infty$-geometry~\citep{balles2020geometry,bernstein2024old,xieICLR2025ellinfty,xie2025tale}. Under this alternative geometry, we provide the first rigorous characterization of a problem class in which SignSGD provably outperforms SGD.

To clarify this geometric shift, it is essential to highlight the distinction between these conditions. Standard $\ell_2$-smoothness assumes the loss landscape curves uniformly across all directions, while $\ell_\infty$-smoothness bounds the gradient change based on the maximum coordinate-wise distance, providing a hypercube-based geometric assumption that aligns with coordinate-wise, scaled updates of sign-based methods. Define $D_f(\x, \y) := f(\y) - (f(\x) + \inner{\nabla f(\x)}{\y-\x})$ as the Bregman Divergence of $f$, the above two smooth models can be expressed as below.
\begin{formulacomparison}
\begin{comparisonpanel}[equal height group=smoothness]{comparisonblue}
    \textbf{$\ell_2$-smoothness}
    \[
    \abs{D_f(\x, \y)} \le \dfrac{L_2}{2} \norm{\y - \x}^2_2
    \]
\end{comparisonpanel}\hfill
\begin{comparisonpanel}[equal height group=smoothness]{comparisongold}
    \textbf{$\ell_\infty$-smoothness}
    \[
    \abs{D_f(\x, \y)} \le \dfrac{\Linf}{2} \norm{\y - \x}^2_\infty
    \]
\end{comparisonpanel}
\end{formulacomparison}
Furthermore, we adopt a separable noise model where variance is tracked independently for each coordinate via a vector $\vsigma = [\sigma_1, \sigma_2, \dots, \sigma_d]$. This fine-grained characterization is more suitable for analyzing the highly imbalanced noise prevalent in modern deep learning~\citep{sagun2017empirical,zhu19anisotropic,zhang2020whyheavytail,pan2022eigencurve,wu2022alignment,panICLR2024momentum}.
\begin{formulacomparison}
\begin{comparisonpanel}[equal height group=noise]{comparisonblue}
    \textbf{Standard noise model}
    \[
    \expect{\left.\norm{\g_{t}^b-\nabla f(\x_t)}_2^2\right|\F_{t-1}}\le \sigma^2
    \]
\end{comparisonpanel}\hfill
\begin{comparisonpanel}[equal height group=noise]{comparisongold}
    \textbf{Separable noise model}, $\forall i\in[d]$
    \[
    \expect{\left.\abs{\g_{t,i}^b-\nabla_i f(\x_t)}^2\right|\F_{t-1}}\le \sigma_{i}^2
    \]
\end{comparisonpanel}
\end{formulacomparison}
Under this refined geometry, we develop a clean, self-contained lower bound analysis for SignSGD. We prove an $\ell_1$-norm lower bound that matches its upper bound, thereby giving a tight characterization of its convergence rate. We further derive an $\ell_1$-norm lower bound for SGD under the same setting. Comparing the \emph{upper bound of SignSGD} with the \emph{lower bound of SGD} reveals a strict separation: when the coordinate-wise noise is sparse or highly heterogeneous, SignSGD achieves a better dimension dependence. Finally, we extend this framework to matrix optimization and establish the first lower bound for Muon measured by the nuclear norm with matching upper bounds.

The main contributions of this paper are summarized as follows.
\begin{itemize}
    \item \textbf{Tight Bounds for SignSGD.} Let $\Linf$ and $\vsigma$ denote the $\ell_\infty$-smooth Lipschitz constant and the noise standard-deviation bound vector, respectively. We establish that the convergence rate of SignSGD with a constant step size is:
    $
        \E \mbrac{\min_{t} \norm{\nabla f(\x_t)}_1} = \Theta\brac{ \sqrt{\nicefrac{L_\infty \Delta}{N}} + \brac{\nicefrac{\norm{\vsigma}_1^2L_{\infty} \Delta}{N}}^{\frac{1}{4}} }
    $.
    
    \item \textbf{Provable Dimensional Gains of SignSGD over SGD.} By comparing SignSGD's \textit{upper bound} with SGD's \textit{lower bound}, we prove that SignSGD achieves a strictly superior dimension dependence. We further show that when the noise $\vsigma$ is sparse, SignSGD's complexity can be $d$ times better than that of SGD. We validate these theoretical findings by demonstrating accelerated convergence during the pretraining of nanoGPT and LLaMA.
    
    \item \textbf{Unified Extension to Muon.} We elevate our theoretical framework into the matrix domain, mapping separable operations to orthogonalized matrix steps. By applying \emph{anisotropic matrix variance} (\cref{ass:Matrix-variance}) and \emph{spectral norm smoothness} (\cref{ass:Spectral norm smooth}), we establish matching upper and lower bounds for the recently proposed Muon optimizer~\citep{jordan2024muon}, proving for the first time that Muon handles large-scale matrix parameters just as effectively as SignSGD handles flat vectors.
\end{itemize}

\section{Related Work\label{sec:related-work}}

\paragraph{Sign-based methods: vector optimizers}
The idea of using only gradient signs dates back at least to RProp~\citep{riedmiller1993rprop}, while the modern stochastic formulation was introduced by~\citet{bernstein2018signsgd}, who studied SignSGD and its momentum variant Signum for smooth non-convex optimization. Since SignSGD replaces the stochastic gradient by its coordinate-wise sign, its natural stationarity measure is $\ell_1$-norm rather than the usual $\ell_2$-norm. This geometry is also closely tied to communication efficiency, as one-bit updates and majority vote make SignSGD attractive in distributed learning~\citep{bernstein2019signsgd,jin2020stochastic,safaryan2021stochastic}.

A fundamental limitation of vanilla sign updates appears in non-smooth optimization. Even for simple convex non-smooth objectives, SignSGD can fail to converge because the discontinuity of the sign map may repeatedly select an unfavorable subgradient direction~\citep{karimireddy2019error,xiao2023stochastic,yu2026finite}. Error feedback repairs this issue by accumulating and correcting compression errors~\citep{karimireddy2019error}. More recently, StoSignSGD~\citep{yu2026stosignsgd} injects unbiased structural stochasticity directly into the sign conversion, using coordinate-wise max-buffer to control the stochastic sign scale, and coupling its noise level with the gradient signal. It preserves the numerical simplicity of sign updates while resolving the non-smooth non-convergence pathology of deterministic SignSGD illustrated in~\citet{karimireddy2019error,yu2026finite}.

In smooth stochastic optimization, a separate line of work sharpens the convergence theory of sign-based methods through momentum-based analyses~\citep{sun2023momentum,jiang2025improved}, stochastic sign mechanisms~\citep{jin2020stochastic,safaryan2021stochastic}, variance reduction~\citep{chzhen2023signsvrg,NeurIPS:2024:Jiang}, and heavy-tailed noise analyses~\citep{kornilov2025signheavytail,yu2026signheavytails}. Sign-based vector optimizers are also closely related to adaptive methods. Several works interpret the optimization dynamics of Adam~\citep{kingma15adam} through signed or normalized update directions~\citep{balles2018dissecting,crawshaw2022robustness,kunstner2023noise,peng2025simple}, while Lion~\citep{chen2023symbolic} has emerged as a representative momentum-based sign optimizer with strong empirical performance and growing theoretical support~\citep{chenICLR2024lion,dong2024convergence,jiang2025lion,sfyraki2025lions,yu2026signheavytails}. ~\citet{yu2026signheavytails} demonstrate that SignSGD- and Lion-type methods can provably outperform AdamW, especially in LLM training regimes where the noise is heavy-tailed~\citep{zhang2020whyheavytail,yu2026openproblem}.

\paragraph{Sign-based methods: matrix optimizers}
Matrix sign optimizers extend the sign-based philosophy from coordinate-wise vector updates to structured matrix updates. The representative example is Muon~\citep{jordan2024muon}, which applies an orthogonalized matrix sign direction: for a matrix gradient $G=U\vSigma V^\top$, the update direction is $\msign{G}=UV^\top$. In practice, this matrix sign operation is efficiently approximated by Newton--Schulz iterations~\citep{kovarik1970some,bjorck1971iterative}. Subsequent variants, such as MuonLight, incorporate Nesterov momentum, learning-rate alignment, and implementation refinements, and have been used or benchmarked in massive-scale language model training~\citep{liu2025muon,zeng2025glm,team2025kimi,wen2025fantastic,semenov2025benchmarking,deepseekai2026deepseekv4}. These works suggest that matrix-level normalization can provide optimization benefits beyond coordinate-wise sign updates.

Theoretical understanding of Muon and related matrix optimizers is still developing. Existing studies analyze Muon from several complementary perspectives, including spectral and nuclear-norm geometry, normalized matrix descent, preconditioning, connections to Adam-type methods, and empirical justifications for its superiority~\citep{li2025note,shen2025convergence,chang2025convergence,si2025adamuon,sfyraki2025lions,chen2025muon,huang2025limuon,li2025normuon,qian2025muon,tveit2025muon,page2025muonall,mehta2025muon,vasudeva2025muon,vasudeva2025the,frans2025really,pan2025unbiased,wang2025muon,zhang2025provable,zhang2025adagrad,su2025isotropic,crawshaw2025exploration,ma2026preconditioning,du2026newton,yang2026memory}. 

\paragraph{Lower complexity bounds}  
Complexity lower bounds are the standard tool for determining whether algorithmic improvement is genuine or merely an artifact of analysis. The oracle-complexity viewpoint goes back to the classical framework of~\citet{nemirovskij1983problem}, and was later developed extensively for stochastic convex optimization~\citep{agarwal2012information}. These results identify how noise, dimension, and geometry constrain the best possible rates in stochastic first-order optimization.

For smooth non-convex optimization, the classical target is an $\epsilon$-stationary point measured by $\|\nabla f(x)\|_2$. In the deterministic setting, \citet{carmon2020lower,carmon2021lower} established sharp lower bounds for finding stationary points of smooth high-dimensional functions, while~\citet{chewi2023complexity} showed that even in one dimension the query complexity depends delicately on whether the algorithm is deterministic or randomized and whether it can access first-order or zeroth-plus-first-order information. In the stochastic setting,~\citet{ghadimi2013stochastic} gave the standard $O(\epsilon^{-4})$ upper bound for SGD under unbiased bounded-variance gradients, and subsequent lower-bound work showed that this rate is essentially unavoidable. In particular,~\citet{drori2020complexity} studied worst-case lower bounds for SGD itself, and~\citet{arjevani2023lower} proved that any stochastic first-order method requires $\Omega(\epsilon^{-4})$ stochastic-gradient queries under the standard bounded-variance model, with a corresponding $\Omega(\epsilon^{-3})$ lower bound under mean-squared smoothness. Thus, in the usual $\ell_2$-geometry, SGD is minimax optimal, and one should not expect a worst-case improvement without changing the geometry, the oracle model, or the stationarity criterion.  
  
A more recent line of work asks whether adaptive or sign-based methods can provably improve over SGD once the problem structure is refined.~\citet{jiang2024convergence} demonstrates the possibility for AdaGrad under coordinate-wise smoothness and coordinate-wise noise assumptions: by adopting an $\ell_1$-stationarity measure, they prove both an AdaGrad upper bound and an SGD lower bound, yielding regimes where AdaGrad improves over SGD by a factor of the dimension. ~\citet{crawshaw2025complexity} study lower bounds for adaptive gradient algorithms under $(L_0,L_1)$-smoothness, showing several AdaGrad variants necessarily incur higher-order dependence on relaxed-smoothness parameters.

\section{SignSGD: Upper Bound and Lower Bound\label{sec:signsgd-bounds}}

In this section, we give the first tight characterization of the convergence of SignSGD.

\subsection{Notations and Assumptions}

We write $[T]$ for $\{1,2,\dots,T\}$, $\abs{\x}$ for element-wise absolute value of $\x\in\R^d$. The $\L$-weighted vector norm for $\L \in\R_{+}^{d}$ is defined as $\norm{\x}_{\L}^{2}:=\x^{\top}\diag(\L)\x$. For vector descent methods, we study the optimization problem $\min_{\x\in\R^d}f(\x)$, where $f:\R^d \to \R$ is differentiable. Given a point $\x\in\R^d$, we can only access the gradient $\nabla f(\x)=[\nabla_1f(\x),\cdots,\nabla_df(\x)]\in\R^d$ in a noisy manner which we will later define in \cref{ass:variance}.  We use $\sign{\cdot}$ to denote the sign operator, and WLOG, use the convention $\sign{0} = 0$. Below, we list some necessary assumptions.

\renewcommand{\theassumption}{\arabic{assumption}a}
\renewcommand{\theHassumption}{\arabic{assumption}a}
\setcounter{assumption}{0} 

\begin{assumption}[Lower bounded objective]
\label{ass:bounded from below}
The function $f$ is bounded from below. There exists $f^* > -\infty$ such that $f(\x) \geq f^*$, for all $\x\in \mathbb{R}^{d}$. We further denote $\Delta = f(\x_0) - \inf_{\x\in\R^d} f(\x)$.
\end{assumption}

\begin{assumption}[Separable noise model]
\label{ass:variance}
At step $t$ we observe a mini-batch of gradients $g_{t}=\{\g_{t}^{1},\cdots,\g_{t}^{B}\}$, mutually independent conditioned on $\F_{t-1}$ and satisfying $\E\sqbrac{\g_t^b|\F_{t-1}}=\nabla f(\x_t),\forall b\in[B]$ where $\mathcal{F}_{t}=\sigma(g_{0},\dots,g_{t})$ denotes the natural filtration, with $\F_{-1}$ trivial. Moreover, the coordinate-wise conditional variance is bounded separably: there exist non-negative constants $\vsigma = [\sigma_1, \sigma_2, \cdots , \sigma_d]$ such that
$$
\expect{\left.\abs{\g_{t,i}^b-\nabla_i f(\x_t)}^2\right|\F_{t-1}}\le \sigma_{i}^2,\quad\forall i\in[d],\forall b\in[B].
$$
\end{assumption}
\cref{ass:bounded from below} is standard and necessary for stochastic non-convex optimization~\citep{arjevani2023lower}. \cref{ass:variance} is widely used in analysis of adaptive methods like AdaGrad~\citep{duchi2011adaptive}, and sign-based methods like SignSGD~\citep{bernstein2018signsgd,JMLR:v26:24-0523,jiang2024convergence,yu2026signheavytails}.

\begin{assumption}[$\ell_\infty$-smoothness]
\label{ass:infty smoothness}
The objective function $f$ is $\ell_\infty$-smooth if there exists non-negative constant $\Linf$ such that for all $\x, \y \in \R^d$, we have
$$
\abs{f(\y) - (f(\x) + \inner{\nabla f(\x)}{\y-\x})} \le \dfrac{\Linf}{2} \norm{\y - \x}^2_\infty.
$$
\end{assumption}

Besides \cref{ass:separable smoothness}, separable smoothness (\cref{ass:separable smoothness}) is also widely used in analyses of SignSGD methods~\citep{bernstein2018signsgd,bernstein2019signsgd,safaryan2021stochastic,yu2026signheavytails}, which naturally aligns with sign descent methods due to their separable nature.

\begin{assumption}[Separable smoothness]
\label{ass:separable smoothness}
The objective function $f$ is $\L$-separable smooth if there exists a non-negative vector $\L = [L_1, L_2, \dots, L_d]$ such that for all $\x, \y \in \R^d$, we have
$$
\abs{f(\y) - (f(\x) + \inner{\nabla f(\x)}{\y-\x})} \le \dfrac{1}{2}\norm{\y - \x}^2_\L.
$$
\end{assumption}

\citet{balles2020geometry} further discussed the relationship between \cref{ass:infty smoothness,ass:separable smoothness}, pointing out that separable smoothness is a sufficient condition for $\ell_\infty$-smoothness. We present the following lemma to formalize this claim, whose proof can be found in \cref{pf:smooth-Equivalence}. 

\begin{lemma}\label{lem:smooth-equivalence}
    If $f$ is $\L$-separable smooth, then it is also $\ell_\infty$-smooth with $\Linf = \norm{\L}_1$.
\end{lemma}

For a unified view of this non-Euclidean smoothness, we refer the readers to~\citet{yu2026mirror}.

\subsection{Upper Bound Theory of SignSGD}
For completeness, the whole algorithmic procedure of SignSGD is listed in \cref{alg:ssgd}. Though the upper bounds of SignSGD and its variants under various settings have been widely studied~\citep{bernstein2018signsgd,karimireddy2019error,sun2023momentum,NeurIPS:2024:Jiang,jiang2025improved,jiang2025lion,yu2026signheavytails}, we present the upper bound analysis here for consistency of the paper. 

\begin{theorem}
\label{thm:ssgd_upper_infty}
    Run \cref{alg:ssgd} for $T$ iterations under \cref{ass:infty smoothness,ass:variance,ass:bounded from below}, with the hyperparameters specified in \eqref{eq:hyp_ssgd}.
    Denote $N = BT$ as the total complexity, \cref{alg:ssgd} guarantees:
    \begin{equation}\label{eq:convergence-rate-sign-infty}
        \expect{\dfrac{1}{T}\sum_{t=0}^{T-1}\norm{\nabla f(\x_t)}_1}
        = \bigO{
        {\sqrt{\dfrac{\Linf\Delta}{N}}} + \brac{\dfrac{\norm{\vsigma}_1^2\Linf \Delta}{N}}^{\frac{1}{4}}
        }.
    \end{equation}
\end{theorem}

\begin{figure}[t]
    \begin{minipage}[t]{0.49\textwidth}
        \begin{algorithm}[H]
        \caption{SignSGD~\citep{bernstein2018signsgd}}
        \phantomsection
        \label{alg:ssgd}
        \begin{algorithmic}[1]
        \STATE \textbf{Require} iteration number $T$ , initial point $\x_0$, batch size $B$, learning rate $\eta$
        \FOR {step $t = 0$ {\bf to} $T-1$}
            \STATE Compute stochastic gradient:\newline $\g_t = \nicefrac{1}{B}\sum_{b=1}^B \g_{t}^{b}$
            \STATE $\x_{t+1} = \x_{t} - \eta \sign{\g_t}$
        \ENDFOR
        \end{algorithmic}
        \end{algorithm}
    \end{minipage}
    \hfill
    \begin{minipage}[t]{0.49\textwidth}
        \begin{algorithm}[H]
        \caption{Muon~\citep{jordan2024muon}}
        \phantomsection
        \label{alg:muon}
        \begin{algorithmic}[1]
        \STATE \textbf{Require} iteration number $T$ , initial point $\WB_0$, batch size $B$, learning rate $\eta$.
        \FOR {step $t = 0$ {\bf to} $T-1$}
            \STATE Compute stochastic gradient:\newline $\Gb_t = \nicefrac{1}{B}\sum_{b=1}^B \Gb_t^b$
            \STATE $\WB_{t+1} = \WB_{t} - \eta \msign{\Gb_t}$
        \ENDFOR
        \end{algorithmic}
        \end{algorithm}
    \end{minipage}
\end{figure}

\subsection{Lower Bound Theory of SignSGD}
After establishing an upper bound for \cref{alg:ssgd}, we present a corresponding lower bound, which immediately verifies the sharpness of the upper bound in \cref{thm:ssgd_upper_infty}.

\begin{theorem}
\label{thm:ssgd_lower_infty}
For any given $\Linf,\Delta>0$, $\vsigma\in\R_+^d$, and deterministic initial point $\x_0\in\R^d$, fix integers $B,T\ge1$ in the parameter range specified in \cref{eq:ssgd_lower_parameter_regime} and an arbitrary constant step size $\eta>0$. There exist a function $f:\R^d\to\R$ and a stochastic gradient oracle such that:
\begin{enumerate}
\setlength{\itemsep}{2pt}\setlength{\parsep}{0pt}\setlength{\parskip}{0pt}
    \item $f$ satisfies \cref{ass:infty smoothness,ass:bounded from below};
    \item the per-sample stochastic gradients $\g_t^b$ satisfy \cref{ass:variance};
    \item denote $N = BT$, the iterates generated by \cref{alg:ssgd} satisfy
    \begin{equation}
    \label{eq:lower_signsgd}
        \expect{\min_{0\le t<T} \norm{\nabla f(\x_t)}_1}
        =
        \Omega\brac{
            \sqrt{\frac{\Linf \Delta}{N}}
            +
            \brac{\frac{\norm{\vsigma}_1^2 \Linf \Delta}{N}}^{\frac{1}{4}}
        } .
    \end{equation}
\end{enumerate}
\end{theorem}
To the best of our knowledge, \cref{thm:ssgd_lower_infty} is the first known tight lower bound for the SignSGD algorithm that does not depend on the dimension $d$ explicitly. We briefly illustrate our proof techniques below, with the full analysis postponed to \cref{sec:proof-ssgd-lowerbound}. 

\begin{proof}[Proof Sketch]
    To analyze the complexity of SignSGD, we divide the proof into four main steps.

    \textbf{Step 1: Dimensional Decomposition.} We first reduce the $d$-dimensional optimization into $d$ parallel one-dimensional problems by constructing a separable objective function $f(\x) = \SumD p_i(\x_i)$, where $\x_i$ denotes the $i$-th coordinate. Utilizing \cref{lem:smooth-equivalence}, we can safely distribute the global $\ell_\infty$-smoothness constant $\Linf$ and the suboptimality $\Delta$ across coordinates such that $\norm{\L}_1 = \Linf$ and $\SumD \Delta_i = \Delta$.

    \textbf{Step 2: Constructing a 1D Resisting Oracle.} We establish the worst-case lower bound for a single coordinate by constructing a ``resisting oracle''~\citep{nesterov2018lectures}. Specifically, in the deterministic setting, we design a hard 1D function $p_i$ that maintains a constant slope $p_i'(x) = -\epsilon$ at a finite sequence of equally spaced query points. Consequently, iterates that remain on this grid have gradient magnitude $\epsilon$.

    \begingroup\widowpenalty=10000\displaywidowpenalty=10000
    \textbf{Step 3: Inducing Stall via Adversarial Bimodal Noise.} To extend the 1D construction to the stochastic setting, we introduce an adversarial noise distribution. At these grid points, we define a specialized gradient oracle
    \begin{equation}\label{eq:noise}
        \Pr \brac{g_t^b = 0 \mid x_t} = \frac{\sigma_i^2}{\sigma_i^2 + \epsilon^2}, \quad \Pr \brac{g_t^b = \frac{\sigma_i^2 + \epsilon^2}{\epsilon^2} p_i'(x_t) \mid x_t} = \frac{\epsilon^2}{\sigma_i^2 + \epsilon^2}.
    \end{equation}
    \endgroup
    At these grid points, the oracle is unbiased with variance $\sigma_i^2$; outside the grid, it returns the exact derivative. Crucially, this specific noise structure is designed to force the optimizer to stall, so that more total steps are needed to escape from the previously defined ``resisting oracle'' when the variance $\sigma_i^2$ is large. This yields the one-dimensional lower bound:
    \begin{align*}
        \expect{\min_t |p'(x_t)|} \ge C\max \left\{\sqrt{\dfrac{L\Delta}{N}}, \brac{\dfrac{L\Delta\sigma^2}{N}}^{\frac{1}{4}}\right\}.
    \end{align*}
    \textbf{Step 4: Dimensional Lifting and Adversarial Tuning.} Finally, we aggregate the 1D lower bounds across all $d$ dimensions. By adversarially tuning the coordinate-wise smoothness ($L_i \propto \sigma_i$) and the suboptimality ($\Delta_i$) subject to their global constraints, we maximize the aggregated lower bound. This step successfully lifts the 1D result to the $d$-dimensional setting and recovers the exact dependency on $\norm{\vsigma}_1$ as stated in the theorem.
\end{proof}

\section{Why Sign Operator Works: Comparison Between SignSGD and SGD}\label{sec:compare}

In this section, we first give the lower bound of SGD, and then compare SignSGD with SGD under theoretical  perspectives. This comparison highlights the role of sign descent beyond its classical interpretation as a gradient-compression mechanism~\citep{bernstein2018signsgd}: under \cref{ass:infty smoothness}, the sign operator can also lead to provably improved complexity bounds.

\subsection{Theoretical Study: SignSGD Converges Provably Faster than SGD}\label{sec:theo_improve}
We introduce the following $\ell_1$-norm lower bound for SGD. This result characterizes the intrinsic hardness faced by SGD under our $\ell_\infty$-smooth geometry, and therefore serves as the key benchmark for comparison with the upper bound of SignSGD.
\begin{theorem}[Mini-batch SGD lower bound]
\label{thm:sgd_lower_infty}
For any given $d\ge1$, $\Linf,\Delta>0$, $\vsigma\in\R_+^d$, and deterministic initial point $\x_0\in\R^d$, fix an arbitrary constant step size $\eta>0$ and consider vanilla mini-batch SGD with batch size $B$ and $T$ iterations in the parameter range specified in \cref{eq:sgd_lower_parameter_regime}. There exists a function $f:\R^d\to\R$ and a globally unbiased stochastic gradient oracle such that:
\begin{enumerate}
\setlength{\itemsep}{2pt}\setlength{\parsep}{0pt}\setlength{\parskip}{0pt}
    \item $f$ satisfies \cref{ass:infty smoothness,ass:bounded from below}, with $f(\x_0)-\inf f=\Delta$;
    \item the per-sample stochastic gradients $\g_t^b$ satisfy \cref{ass:variance};
    \item denote $N=BT$, the iterates satisfy
\begin{equation}\label{eq:lower_sgd_infty}
\expect{\min_{0\le t<T}\norm{\nabla f(\x_t)}_1}
=\Omega\brac{\sqrt{\frac{d\Linf\Delta}{N}}+
\brac{\frac{d\norm{\vsigma}_2^2\Linf\Delta}{N}}^{\frac14}}.
\end{equation}
\end{enumerate}
\end{theorem}

The proof is deferred to \cref{sec:proof-sgd-lower-bound}.

\paragraph{Comparison based on concentration coefficient $\chi(\vsigma)$.} We proceed to compare the rate obtained in \cref{thm:ssgd_upper_infty}, which is the \emph{upper bound} for SignSGD, and \cref{thm:sgd_lower_infty}, which is the \emph{lower bound} for SGD. We define the concentration coefficient $\chi : \R^d \setminus \{\mathbf{0}\} \to [1,d]$ as follows:
\begin{equation}
    \chi(\vsigma) = \dfrac{d\norm{\vsigma}_2^2}{\norm{\vsigma}_1^2} \in \left[1, d \right].\label{eq:chi-sigma}
\end{equation}
The above function $\chi$ is a measure of how ``concentrated'' a vector is. $\chi(\vsigma) \approx d$ corresponds to highly skewed noise acting on only a few coordinates (with $\chi(\vsigma) = d$ for a single coordinate), whereas $\chi(\vsigma) = 1$ represents perfectly uniform noise across all parameters. To find an $\epsilon$-stationary point in the $\ell_1$-norm, the sample-complexity bounds are
\begin{formulacomparison}
\begin{comparisonpanel}[equal height group=complexity]{comparisonblue}
    \textbf{SignSGD}
    \begin{equation}\label{eq:sign_complexity}
    \Theta \brac{\dfrac{\Linf \Delta}{\epsilon^2} + \dfrac{\norm{\vsigma}_1^2 \Linf \Delta}{\epsilon^4}}
    \end{equation}
\end{comparisonpanel}\hfill
\begin{comparisonpanel}[equal height group=complexity]{comparisongold}
    \textbf{SGD}
    \begin{equation}\label{eq:sgd_complexity}
    \Omega\left(\dfrac{\textcolor{red}{d}\Linf \Delta}{\epsilon^2} + \dfrac{\textcolor{red}{\chi(\vsigma)}\norm{\vsigma}_1^2 \Linf \Delta}{\epsilon^4}\right)
    \end{equation}
\end{comparisonpanel}
\end{formulacomparison}
In view of the above two complexity bounds, we make the following observations.

\begin{enumerate}
    \item \textbf{Deterministic term.} For the first noiseless term, SignSGD reduces the complexity of SGD by a factor of $d$, revealing its superior dependence on the dimensional factors.

    \item \textbf{Stochastic term.} Since $1 \le \chi(\vsigma) \le d$, the displayed stochastic terms coincide at $\chi(\vsigma) = 1$, while their ratio reaches $d$ at $\chi(\vsigma) = d$ (one-coordinate noise).
\end{enumerate}

As shown above, our results provide the first problem setting in which SignSGD achieves a provably better complexity bound (in terms of the dimensional dependence) than SGD.

%---Matrix---%

\section{Matrix Optimizers}

While vector-based sign methods effectively handle separable scaling, modern LLM architectures rely heavily on matrix multiplications. This motivates extending our framework to the matrix domain. Now that we have already illustrated the edge of SignSGD in \cref{sec:compare}, we will show in the sequel that it's natural to extend our previous framework from vectors to the matrix domain.

\subsection{Notations and Assumptions}

Throughout this section,  we consider the matrix optimization problem $\min_{\R^{m\times n}} F(\WB)$ under the following assumptions. We suppose that matrices are oriented so that $m\le n$.

We denote the set of $m \times m$ positive semi-definite (PSD) matrices by $\SBB^m$. For a matrix $\XB \in \R^{m \times n}$ of rank $r$, we write its compact singular value decomposition (SVD) as $\XB = \U\bSigma\V^\top$, where $\U \in \R^{m \times r}$ and $\V \in \R^{n \times r}$ have orthonormal columns and $\bSigma = \diag(\sigma_1,\dots,\sigma_r)$ retains only the positive singular values $\sigma_1 \ge \dots \ge \sigma_r > 0$. The matrix sign operator is defined as $\msign{\XB} := \U\V^\top$, with $\msign{\mathbf{0}} := \mathbf{0}$. In this section, we consider pure polar-update Muon, which uses no Newton--Schulz iterations or momentum. The inner product between matrices $\XB,\YB\in\R^{m\times n}$ is denoted by $\inner{\YB}{\XB}:=\Tr{\YB^\top\XB}$. We let $\norm{\cdot}_\op$, $\norm{\cdot}_*$, and $\norm{\cdot}_\Fn$ denote the matrix operator norm, nuclear norm, and Frobenius norm, respectively. For a vector $\mathbf{v} \in \mathbb{R}^m$, $\text{diag}(\mathbf{v}) \in \mathbb{R}^{m \times m}$ denotes the diagonal matrix whose elements are given by $(\text{diag}(\mathbf{v}))_{ij} = v_i$ if $i = j$ and $0$ otherwise. 

\renewcommand{\theassumption}{\arabic{assumption}b}
\renewcommand{\theHassumption}{\arabic{assumption}b}
\setcounter{assumption}{0} 

\begin{assumption}[Lower bounded objective]
\label{ass:Matrix-bounded-below}
    The function $F: \R^{m \times n} \to \R$ is bounded from below. There exists $F^* > -\infty$ such that $F(\WB) \geq F^*$ holds for all $\WB \in \mathbb{R}^{m \times n}$. We further denote $\Delta = F(\WB_0) - \inf_{\WB\in\R^{m \times n}} F(\WB)$.
\end{assumption}
 
\begin{assumption}[Anisotropic matrix variance bound]
\label{ass:Matrix-variance}
    At step $t$ we observe a mini-batch of gradients $G_{t}=\{\Gb_{t}^{1},\cdots,\Gb_{t}^{B}\}$, mutually independent conditioned on $\F_{t-1}$ and satisfying $\E\sqbrac{\Gb_t^b|\F_{t-1}}=\nabla F(\WB_t),\forall b\in[B]$ where $\mathcal{F}_{t}=\sigma(G_{0},\dots,G_{t})$ denotes the natural filtration, with $\F_{-1}$ trivial. Denote by $\NB_t^b = \Gb_t^b - \nabla F(\WB_t)$, there exists $\vSigma \in \SBB^{m}$ such that
    \begin{align*}
        \expect{(\NB_t^b) (\NB_t^b)^\top|\F_{t-1}} \preceq \vSigma^2,\quad\forall b\in[B].
    \end{align*}
\end{assumption}

\begin{assumption}[Spectral norm smoothness]
\label{ass:Spectral norm smooth}
    We say $F: \R^{m \times n} \to \R$ is $L_*$-spectral norm smooth if for any $\WB, \WB' \in \R^{m \times n}$, it holds that
    \begin{align*}
        \abs{F(\WB') - (F(\WB) + \inner{\nabla F(\WB)}{\WB'-\WB})} \le \dfrac{L_*}{2} \norm{\WB' - \WB}^2_{\op}.
    \end{align*}
\end{assumption}
\cref{ass:Matrix-variance} can be viewed as an extension of \cref{ass:variance} into matrix space, and has been widely used for analyzing Muon and other matrix optimizers~\citep{an2025asgo, pan2025unbiased}. \cref{ass:Spectral norm smooth} accounts for the distinct structure of matrix parameters, which can be viewed as an extension of \cref{ass:infty smoothness} into matrix space. This extension also connects SignSGD with Muon, as \citet{bernstein2024old} pointed out that SignSGD can be regarded as \textit{steepest descent under infinity norm}, while Muon can be regarded as \textit{steepest descent under spectral norm}.

\subsection{Upper Bound Theory of Muon}
The Muon algorithm is presented in \cref{alg:muon}. We state its upper bound as below, with the proof deferred to \cref{sec:proof-muon-upperbound}.

\begin{theorem}[Muon upper bound]
\label{thm:muon_upper_infty}
    Run \cref{alg:muon} for $T$ iterations under \cref{ass:Spectral norm smooth,ass:Matrix-variance,ass:Matrix-bounded-below}, with the hyperparameters specified in \eqref{eq:hyp_muon}.
    Denote $N = BT$ as the total complexity, we have:
    \begin{equation}\label{eq:convergence-rate-ssgd-infty}
        \expect{\dfrac{1}{T}\sum_{t=0}^{T-1}\norm{\nabla F(\WB_t)}_*}
        = \bigO{
        {\sqrt{\dfrac{L_*\Delta}{N}}} + \brac{\dfrac{\norm{\vSigma}_*^2L_* \Delta}{N}}^{1/4}
        }.
    \end{equation}
\end{theorem}

A key feature of \cref{thm:muon_upper_infty} is that the bound does not depend explicitly on the matrix dimension $\min\{m,n\}$, in sharp contrast to existing analyses of Muon that incur such dimension dependence~\citep{li2025note,shen2025convergence,chang2025convergence,huang2025limuon}.

\subsection{Lower Bound Theory of Muon}
After establishing the upper bound for \cref{alg:muon}, we present a corresponding lower bound, which verifies the sharpness of \cref{thm:muon_upper_infty}. The proof is deferred to \cref{sec:proof-muon-lowerbound}.
\begin{theorem}[Muon lower bound]
\label{thm:muon_lower_infty}
For any given $L_*,\Delta>0$, $\vSigma\succeq0$, and deterministic initial point $\WB_0\in\R^{m\times n}$, fix integers $B,T\ge1$ in the parameter range specified in \cref{eq:muon_lower_parameter_regime} and constant step size $\eta>0$. There exist a function $F:\R^{m\times n}\to\R$ and a stochastic gradient oracle such that:
\begin{enumerate}
\setlength{\itemsep}{2pt}\setlength{\parsep}{0pt}\setlength{\parskip}{0pt}
    \item $F$ satisfies \cref{ass:Spectral norm smooth,ass:Matrix-bounded-below};
    \item the per-sample stochastic gradients $\Gb_t^b$ satisfy \cref{ass:Matrix-variance};
    \item denote $N = BT$, the iterates generated by \cref{alg:muon} satisfy
    \begin{equation}
    \label{eq:lower_muon}
        \expect{\min_{0\le t<T} \norm{\nabla F(\WB_t)}_*} =
        \Omega\brac{
            \sqrt{
                \dfrac{L_*\Delta}{N}
            } 
            + \brac{
                \dfrac{\norm{\vSigma}_*^2 L_*\Delta}{N}
            }^{1/4}
        } .
    \end{equation}
\end{enumerate}
\end{theorem}
To our best knowledge, \cref{thm:muon_lower_infty} is the first lower complexity bound for Muon. Existing stochastic lower bounds, such as~\citet{arjevani2023lower}, are formulated for vector-valued methods under Euclidean geometry, and thus do not directly apply to spectral-norm smoothness and nuclear-norm stationarity. Our theorem fills this gap and shows that the dimension-free upper bound in \cref{thm:muon_upper_infty} is unimprovable in our matrix geometry. We outline the proof below.

\begin{proof}[Proof Sketch]
    Our proof establishes a strict geometric and dynamic equivalence, mathematically reducing the Muon optimization in the matrix domain to the SignSGD optimization in the vector domain. We outline the proof in the following five steps.

    \textbf{Step 1: Constructing the Matrix Objective with Orthogonal Alignments.} We lift the hard vector instance $f(\x)$ into the matrix domain by defining $F(\WB) = \sum_{i=1}^m f_i\big((\QB\WB\PB^\top)_{ii}\big)$. The reasons why we additionally introduce the orthogonal matrix $\QB$ and the projection matrix $\PB$ here are:
    \begin{itemize}
        \item $\QB$ is explicitly used to align the separable vector noise with the target noise-scale matrix $\vSigma$;
        \item $\PB$ is used as the projection matrix, projecting the $m \times n$ matrix into an $m \times m$ subspace. 
    \end{itemize}
    \textbf{Step 2: Designing the Structured Matrix Oracle.} To mirror the stochasticity, we construct a structured stochastic matrix gradient $\Gb_t = \QB^\top \diag(\g_t) \PB$. Driven by the alignment in Step 1, this specific design perfectly satisfies the required noise-scale matrix $\vSigma = \QB^\top \diag(\vsigma) \QB$, while safely isolating the worst-case vector noise $\g_t$ strictly within the projected subspace.

    \textbf{Step 3: Establishing Trajectory Equivalence via SVD.} We factor the constructed gradient over the active coordinates $S = \{i : g_{t,i} \neq 0\}$. Let $\EB_S$ be the coordinate selector collecting the corresponding standard basis vectors, and define $\U = \QB^\top \EB_S \diag(\sign{\g_{t,S}})$, $\bSigma = \diag(\abs{\g_{t,S}})$, and $\V = \PB^\top \EB_S$, where $\g_{t,S}$ denotes the subvector of $\g_t$ indexed by $S$. The diagonal entries of $\bSigma$ are strictly positive and $\U$ and $\V$ have orthonormal columns, so $\Gb_t = \U \bSigma \V^\top$ is the thin SVD up to the ordering of the columns. The SVD orthogonalization therefore acts as a scalar $\sign{\cdot}$ operator on the active diagonal entries: $\msign{\Gb_t} = \QB^\top \EB_S \diag(\sign{\g_{t,S}}) \EB_S^\top \PB = \QB^\top \diag(\sign{\g_t}) \PB$, with $\msign{\Gb_t} = \mathbf{0}$ when $S$ is empty. Consequently, Muon's matrix update $\WB_{t+1} = \WB_t - \eta \msign{\Gb_t}$ mathematically collapses into the exact SignSGD vector update: $\x_{t+1} = \x_t - \eta \sign{\g_t}$.

    \textbf{Step 4: Aligning Geometric Constraints.} We establish a rigorous one-to-one mapping between the governing metrics of both domains. We prove that the vector $\ell_\infty$-smoothness flawlessly translates to the matrix spectral norm smoothness ($L_* = \Linf$), and the trace norm of the noise-scale matrix matches the $\ell_1$-norm of the vector noise ($\norm{\vSigma}_* = \norm{\vsigma}_1$).

    \textbf{Step 5: Transferring the Complexity Bound.} Finally, because orthogonal transformations preserve singular values, we show that the trace norm (nuclear norm) of the matrix gradient $\norm{\nabla F(\WB_t)}_*$ is precisely equivalent to the $\ell_1$-norm of the vector gradient $\norm{\nabla f(\x_t)}_1$, enabling us to directly invoke the vector bounds from \cref{thm:ssgd_lower_infty} to establish the exact $\Omega(\cdot)$ lower bound for Muon.
\end{proof}

\section{Empirical Study}
We examine controlled numerical examples, optimizer comparisons on image classification and language modeling, and direct coordinate-noise measurements.

\subsection{Experiments on Numerical Examples}

To amplify the difference between the two optimizers, we manually crafted two numerical example functions that fit into the extreme cases stated in \cref{sec:theo_improve}. The experimental results of these numerical examples are shown in \cref{fig:sign_vs_sgd_toy}. Our findings are summarized as follows.

\begin{figure}[htbp]
    \centering
    \begin{subfigure}[b]{0.45\textwidth}
        \centering
        \includegraphics[width=\textwidth]{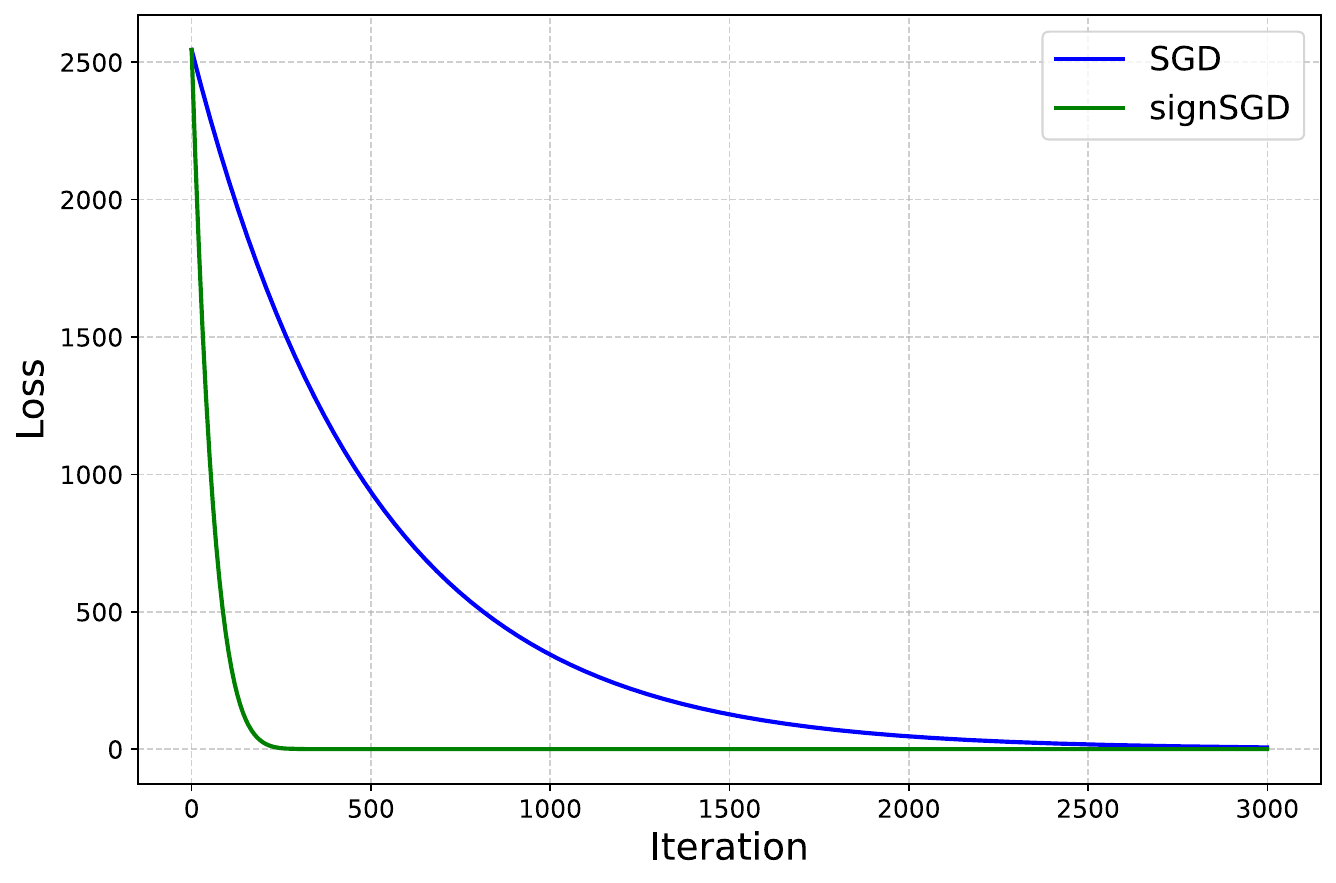}
        \caption{Deterministic case with anisotropic curvature.}
        \label{fig:deter}
    \end{subfigure}
    \hfill
    \begin{subfigure}[b]{0.45\textwidth}
        \centering
        \includegraphics[width=\textwidth]{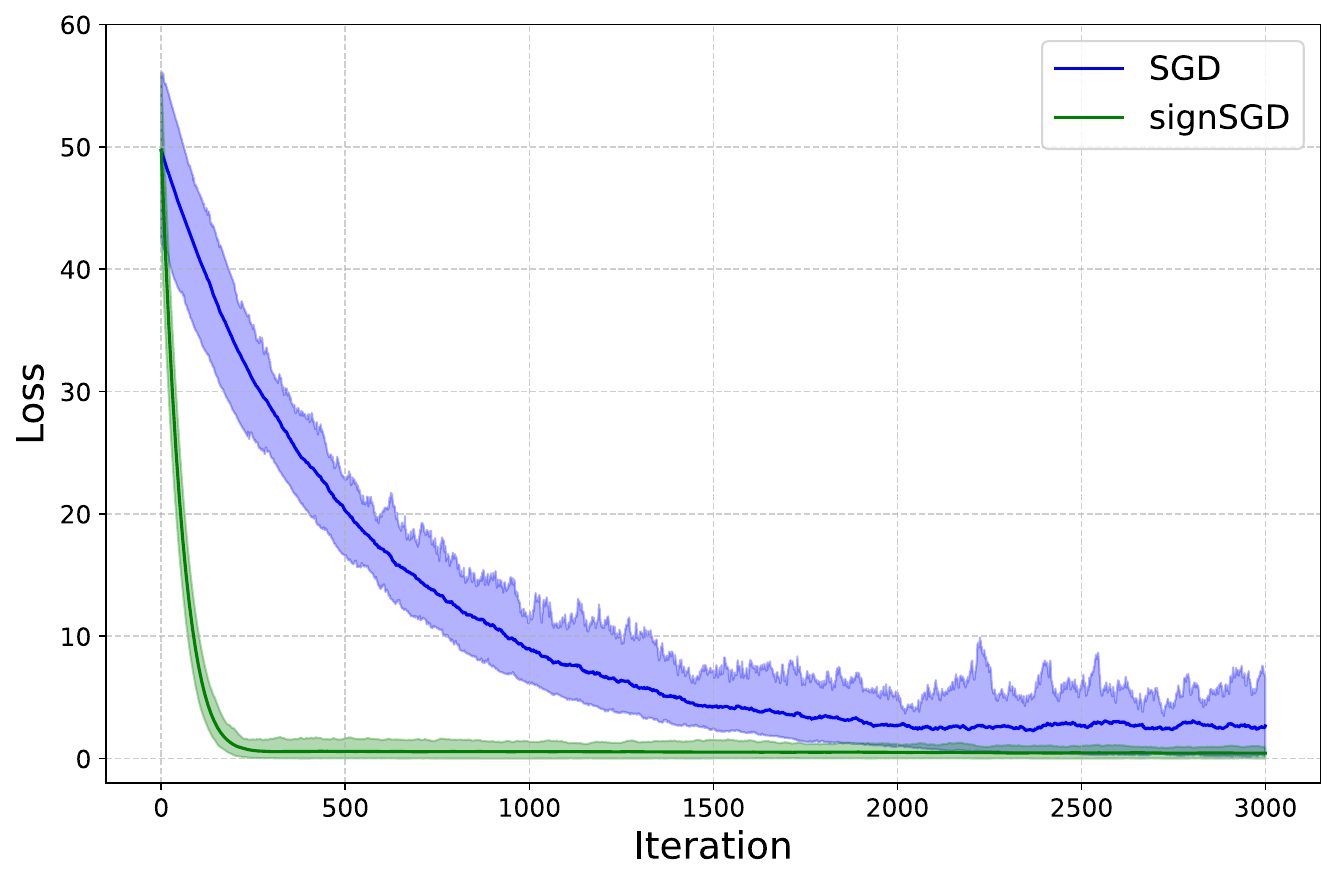}
        \caption{Stochastic example with sparse noise.}
        \label{fig:sto}
    \end{subfigure}
    \caption{Training loss of SignSGD and SGD on the two quadratic examples.}
    \label{fig:sign_vs_sgd_toy}
\end{figure}

\begin{enumerate}
    \item \textbf{Deterministic example.} SignSGD outperforms SGD when facing a $\ell_\infty$-geometry target and a sufficiently large $d$. Therefore, we set our target function as an imbalanced quadratic objective $f(\x) = \nicefrac{1}{2} \sum_{i=1}^{d} L_i x_i^2$ for $ \x \in \R^d$. To better illustrate the difference, we set $d=5000$ and $L_1 = 1000, L_i = 1$ for $ i \ge 2$. From \cref{fig:deter}, we can observe that SGD performs worse than SignSGD because its learning rate is constrained by the largest coordinate curvature $L_1$.

    \item \textbf{Stochastic example.} SignSGD outperforms SGD when facing sparse noise structure. We switch to a simple  quadratic objective function $f(x) = \norm{\x}_2^2/2$ for $\x\in\mathbb{R}^{100}$. Then we add Gaussian noise $\mathcal{N}(0,100^2)$ only to the first component of the gradient, which further maximizes the sparsity of $\vsigma$ as $\chi(\vsigma) = d$ in this setting. Results in \cref{fig:sto} align with our findings, as SignSGD consistently outperforms SGD facing sparse noise.
\end{enumerate}

\begin{figure}[!hbtp]
  \centering
  \begin{subfigure}[t]{0.32\textwidth}
    \centering
    \includegraphics[width=\linewidth]{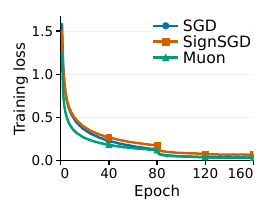}
    \caption{ResNet-20/CIFAR-10}
    \label{fig:empirical-summary-cnn}
  \end{subfigure}\hfill
  \begin{subfigure}[t]{0.32\textwidth}
    \centering
    \includegraphics[width=\linewidth]{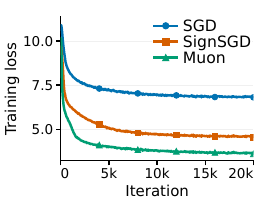}
    \caption{nanoGPT/C4}
    \label{fig:empirical-summary-nanogpt}
  \end{subfigure}\hfill
  \begin{subfigure}[t]{0.32\textwidth}
    \centering
    \includegraphics[width=\linewidth]{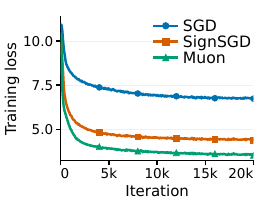}
    \caption{LLaMA/C4}
    \label{fig:empirical-summary-llama130m}
  \end{subfigure}
  \caption{Training losses of SGD, SignSGD, and Muon under different settings.}
  \label{fig:empirical-summary}
\end{figure}

\subsection{Experiments on CNNs and LLM Pretraining}
\label{sec:performance-comparison}
We compare SGD, SignSGD~\citep{bernstein2018signsgd}, and practical Muon~\citep{jordan2024muon} on ResNet-20~\citep{he2016resnet} trained on CIFAR-10~\citep{krizhevsky2009learning}, and on nanoGPT~\citep{nanogpt} and LLaMA trained on C4~\citep{JMLR:v21:20-074}. Each configuration uses three training seeds and the same training budget. In \cref{fig:empirical-summary}, we observe that SignSGD and Muon reach lower training losses than SGD, while the three optimizers attain similar results for the CNN model. Complete curves, final metrics, and implementation differences are reported in \cref{app:performance-configurations}.

\Needspace{12\baselineskip}
\subsection{Direct Coordinate-Noise Measurements}
\label{sec:direct-coordinate-noise}
\begin{wrapfigure}[16]{r}{0.50\textwidth}
  \centering
  \includegraphics[width=\linewidth]{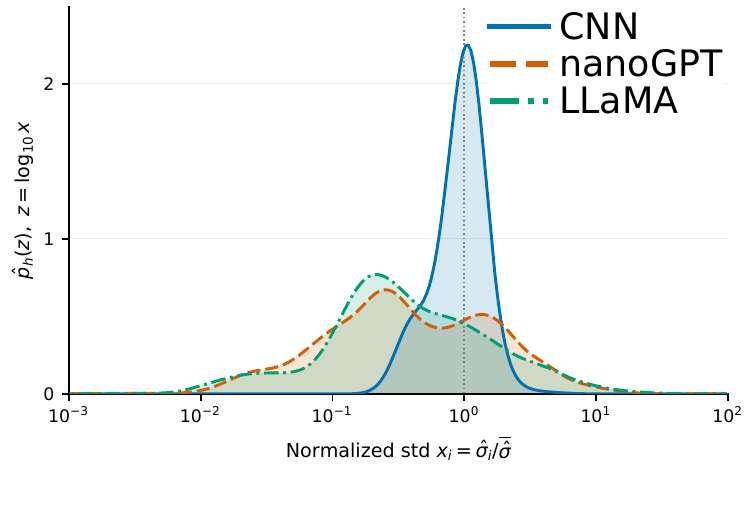}
  \caption{SGD coordinate-wise noise std distributions, estimated at 50\% of training.}
  \label{fig:sgd-sigma-distribution}
\end{wrapfigure}
To understand the contrasting performance of SGD and SignSGD across CNN and LLM training, we compare the coordinate-wise gradient-noise variances along the corresponding SGD trajectories.

As illustrated in \cref{fig:sgd-sigma-distribution}, the coordinate-noise distributions differ between CNN and LLM training with SGD. ResNet-20's noise becomes more homogeneous during training, whereas nanoGPT and LLaMA exhibit stronger coordinate-wise variance heterogeneity and correspondingly a larger estimated concentration coefficient $\widehat\chi$ at all five measured stages. Measurement details are reported in \cref{app:coordinate-noise-measurement}.

These findings provide a noise-geometry interpretation of the contrasting optimization behavior in \cref{fig:empirical-summary}. SGD's update magnitudes depend on gradient magnitudes, so coordinates with larger noise scales can dominate its stochastic updates. SignSGD uses equal-magnitude coordinate-wise sign updates~\citep{bernstein2018signsgd}, avoiding this direct dependence on gradient magnitudes. The more homogeneous CNN noise leaves less scope for this noise-dependent advantage, and SGD attains lower training loss than SignSGD. In LLM pretraining, the stronger variance heterogeneity accompanies lower training losses for SignSGD than for SGD. Our theory in \cref{sec:theo_improve} identifies larger $\chi(\vsigma)$ with a greater noise-dependent complexity advantage for SignSGD, connecting the noise distributions and update rules to the observed cross-model performance contrast.

\section{Conclusion}
In this work, we bridge the longstanding gap between the empirical superiority of sign-based optimizers and their theoretical guarantees in non-convex stochastic optimization. By moving from the standard $\ell_2$ framework to an $\ell_\infty$-smooth, coordinate-wise noise setting with $\ell_1$-stationarity, we give an optimal complexity-theoretic characterization of SignSGD and identify a problem geometry in which it provably outperforms SGD. Specifically, we prove matching upper and lower bounds for SignSGD, establish a strictly worse lower bound for SGD, and show that the resulting separation becomes most pronounced under sparse or highly heterogeneous noise. We further extend this framework to matrix optimization and obtain matching bounds for Muon under anisotropic matrix variance, spectral norm smoothness, and nuclear norm stationarity. Our empirical results corroborate these theoretical findings, with both SignSGD and Muon outperforming SGD in nanoGPT and LLaMA pretraining. Controlled sparse noise toy problems and direct coordinate-noise measurements during LLM pretraining reveal the skewed gradient noise structure leveraged by our theory, supporting the viewpoint that sign-based methods enjoy a genuine geometric advantage in real-world high-dimensional problems, such as LLM pretraining.

\begin{ack}
This work was partially supported by NSFC (U23A20382), the ``111 Center'' (No. B26023), and the Fundamental and Interdisciplinary Disciplines Breakthrough Plan of the Ministry of Education of China (No. JYB2025XDXM118). Dingzhi Yu is supported by the Postgraduate Research \& Practice Innovation Program of Jiangsu Province (NO.26CXJH0364).
The authors declare no competing interests.
\end{ack}

\bibliographystyle{plainnat}
\bibliography{neurips_2026_ref}

\clearpage
\appendix

% \crefalias{section}{appendix} % uncomment if you are using cleveref
\crefalias{section}{appendix}
\crefalias{subsection}{appendix}
\crefalias{subsubsection}{appendix}

\section{Proof of Lemma~\ref{lem:smooth-equivalence}} \label{pf:smooth-Equivalence}

Suppose $f$ is $\L$-separable smooth, we have
\begin{align*}
    \abs{f(\y) - (f(\x) + \inner{\nabla f(\x)}{\y-\x})} &\le \dfrac{1}{2}\norm{\y - \x}^2_\L = \dfrac{1}{2}\SumD L_i (\y_i - \x_i)^2\\
    &\le \dfrac{1}{2}\SumD L_i \norm{\y - \x}_\infty^2= \dfrac{\norm{\L}_1}{2}\norm{\y - \x}_\infty^2.\\
\end{align*}
Therefore $f$ also satisfies \cref{ass:infty smoothness} with $\Linf = \norm{{\L}}_1$.

\section{Analysis for SignSGD}\label{app:sign}

\subsection{Upper Bound for SignSGD}

\begin{proof}[Proof of \cref{thm:ssgd_upper_infty}]
    Set the hyperparameters to
    \begin{equation}\label{eq:hyp_ssgd}
        \eta = \sqrt{\dfrac{2\Delta}{\Linf T}},\quad B = \left\lceil\max \bbrac{1, \dfrac{\norm{\vsigma}_1^2}{\Delta \Linf}T}\right\rceil.
    \end{equation}
    The batch size in \eqref{eq:hyp_ssgd} grows with $T$ when $\norm{\vsigma}_1>0$. At fixed $B$, \eqref{eq:sgd_upper_bound} retains the term $2\norm{\vsigma}_1/\sqrt{B}$, so this theorem is not a constant-batch guarantee for arbitrarily small stationarity.

    Under  \cref{ass:infty smoothness}, we have:
    \begin{align*}
        f(\x_{t+1})  
        &\le &&f(\x_t) + \inner{\nabla f(\x_t)}{\x_{t+1} - \x_t} + \dfrac{\Linf}{2}\norm{\x_{t+1} - \x_t}^2_\infty\\
        &\le &&f(\x_t) -\inner{\nabla f(\x_t)}{\eta \sign{\g_t}} + \dfrac{\eta^2 \Linf}{2}\\
        &= &&f(\x_t) + \inner{\nabla f(\x_t)}{\eta(\sign{\nabla f(\x_t)} - \sign{\g_t})} \\
        & &&- \inner{\nabla f(\x_t)}{\eta \sign{\nabla f(\x_t)}} + \dfrac{\eta^2 \Linf}{2}\\
        &= &&f(\x_t) + \inner{\nabla f(\x_t)}{\eta(\sign{\nabla f(\x_t)} - \sign{\g_t})} - \eta \norm{\nabla f(\x_t)}_1 + \dfrac{\eta^2 \Linf}{2}\\
        &\le &&f(\x_t) + 2\eta \norm{\nabla f(\x_t) - \g_t}_1 - \eta \norm{\nabla f(\x_t)}_1 + \dfrac{\eta^2 \Linf}{2},
    \end{align*}
    where the last inequality is due to
    \begin{align*}
        &\inner{\nabla f(\x_t)}{\eta(\sign{\nabla f(\x_t)} - \sign{\g_t})} \\
        &= \SumD \nabla_i f(\x_t) * \eta(\sign{\nabla_i f(\x_t)} -\sign{\g_{t,i}})\\
        &\le \SumD 2\eta\abs{\nabla_i f(\x_t)} * \mathbb{I}(\sign{\nabla_i f(\x_t)} \neq\sign{\g_{t,i}})\\
        &\le 2\eta \SumD |\nabla_i f(\x_t)-\g_{t,i}| * \mathbb{I}(\sign{\nabla_i f(\x_t)} \neq\sign{\g_{t,i}})\\
        &\le 2\eta \SumD|\nabla_i f(\x_t)-\g_{t,i}| = 2\eta \norm{\nabla f(\x_t) - \g_t}_1.
    \end{align*}
    Rearranging the obtained relation and summing up yields
    \begin{equation}\label{eq:bound_sign}
        \expect{\dfrac{1}{T}\sum_{t=0}^{T-1}\norm{\nabla f(\x_t)}_1} \le \dfrac{\Delta}{\eta T} + 2\expect{\dfrac{1}{T}\sum_{t=0}^{T-1} \norm{\nabla f(\x_t) - \g_t}_1} + \dfrac{\eta \Linf}{2}.
    \end{equation}
    we denote $\xi_t = \nabla f(\x_t) - \g_t$ as the noise of the gradient at iteration $t$. Conditioned on the history $\F_{t-1}$, the batch gradients $\g_t^1,\dots,\g_t^B$ are mutually independent and each is centered at $\nabla f(\x_t)$, so that:
    \begin{equation}\label{eq:variance-reduction-vector}
        \begin{aligned}
        \expect{\dfrac{1}{T}\sum_{t=0}^{T-1} \norm{\nabla f(\x_t) - \g_t}_1}
        &=\expect{\dfrac{1}{T}\sum_{t=0}^{T-1} \norm{ \xi_t}_1} = \dfrac{1}{T}\sum_{t=0}^{T-1}\expect{\norm{\xi_t}_1}\\
        &= \dfrac{1}{T}\sum_{t=0}^{T-1} \SumD \E[|\xi_{t,i}|] \le \dfrac{\norm{\vsigma}_1}{\sqrt{B}},
    \end{aligned}
    \end{equation}
    where the inequality follows from Jensen's inequality for the variance of the mini-batch mean. Putting~\eqref{eq:variance-reduction-vector} back to~\eqref{eq:bound_sign}, and setting $\eta = \sqrt{\nicefrac{2\Delta}{\Linf T}}$, We get:
    \begin{equation}\label{eq:sgd_upper_bound}
        \expect{\dfrac{1}{T}\sum_{t=0}^{T-1}\norm{\nabla f(\x_t)}_1} \le \sqrt{\dfrac{2\Delta \Linf}{T}} + 2\dfrac{\norm{\vsigma}_1}{\sqrt{B}}.
    \end{equation}
    If
    \begin{equation}
        \norm{\vsigma}_1^2 T \le \Delta \Linf,
    \end{equation}
    by setting $B=1$, we have
    \begin{align*}
        \expect{\dfrac{1}{T}\sum_{t=0}^{T-1}\norm{\nabla f(\x_t)}_1} \le \sqrt{\dfrac{2\Delta \Linf}{T}} + 2\dfrac{\norm{\vsigma}_1}{\sqrt{B}} \le (2+\sqrt{2}) \sqrt{\dfrac{\Delta \Linf}{N}}.
    \end{align*}
    Otherwise, we set $B = \left\lceil\nicefrac{\norm{\vsigma}_1^2}{\Delta \Linf}T\right\rceil \le \nicefrac{2\norm{\vsigma}_1^2}{\Delta \Linf}T$, which implies
    \begin{equation}\label{eq:sgd_noise_bound}
        \sqrt{\dfrac{2\Delta \Linf}{T}} \le \brac{\dfrac{8\norm{\vsigma}_1^2\Delta \Linf}{N}}^{\frac{1}{4}}.
    \end{equation}
    Plugging~\eqref{eq:sgd_noise_bound} into~\eqref{eq:sgd_upper_bound}, we get
    \begin{equation}
        \expect{\dfrac{1}{T}\sum_{t=0}^{T-1}\norm{\nabla f(\x_t)}_1} \le \brac{\dfrac{8\norm{\vsigma}_1^2\Delta \Linf}{N}}^{\frac{1}{4}} + 2\brac{\dfrac{\norm{\vsigma}_1^2\Delta \Linf}{N}}^{\frac{1}{4}}.
    \end{equation}
    Combining the above two cases together, we have
    \begin{align*}
        \expect{\dfrac{1}{T}\sum_{t=0}^{T-1}\norm{\nabla f(\x_t)}_1} = \bigO{{\sqrt{\dfrac{\Delta \Linf}{N}}} + \dfrac{(\norm{\vsigma}_1)^{\frac{1}{2}} (\Delta \Linf)^{\frac{1}{4}}}{(N)^{\frac{1}{4}}}}.
    \end{align*}
    This completes the proof.
\end{proof}

\subsection{Lower Bound for SignSGD\label{sec:proof-ssgd-lowerbound}}

The parameter range in \cref{thm:ssgd_lower_infty} is
\begin{equation}\label{eq:ssgd_lower_parameter_regime}
    B,T\in\mathbb N_{\ge1},\qquad N=BT,\qquad
    N\ge\max\left\{1,\frac{2\norm{\vsigma}_1^2}{\Linf\Delta}\right\}.
\end{equation}

To exploit the separable nature of SignSGD before generalizing to $d$ dimensions, we start by bounding the complexity of \cref{alg:ssgd} in finding stationary points in one dimension:

\begin{lemma}
    \label{lemma:ssgd:1}
    For any positive integer $n$, suppose that $\epsilon$ satisfies
    \begin{equation}
        \label{eq:lem-ssgd-2-ass}
        \epsilon \le \dfrac{1}{2\sqrt{n}}
    \end{equation}
    Let $x_1 \in \R$ and $x_t = x_1 + (t-1)\eta$ for any $2 \le t \le n + 1$.  Then there exists a function $p : \R \to \R$ such that: (i) $p$ has a $1$‑Lipschitz gradient; (ii) $p(x_1) - \inf p \le 1$ (iii) $p'(x_t) = -\epsilon$ for any $t \in [1, n]$.
\end{lemma}

\begin{proof}
    We construct the function $p$ as:
    \begin{align*}
        p(x) = 
        \begin{cases}
            -\epsilon (x- x_1) & x\in (-\infty, x_1];\\
            \chi_{t,\epsilon}(x) + c_t & x \in (x_t, x_{t+1}]; \\
            \dfrac{1}{2}(x-x_{n+1})^2 - \epsilon(x-x_{n+1}) + c_{n+1}& x \in (x_{n+1}, \infty),
        \end{cases}
    \end{align*}
    where the values $\{c_t\}_{t=1}^{n+1}$ are chosen to ensure that the function $p$ is continuous. The function $\chi_{t, \epsilon}(x)$ is defined as below:
    \begin{align*}
        \chi_{t, \epsilon}(x) =
        \begin{cases}
            \dfrac{1}{2}(x-x_t)^2 - \epsilon(x - x_t) & x\in \left(x_t, \dfrac{x_t + x_{t+1}}{2}\right];\\[2ex]
            \begin{aligned}
                &-\dfrac{1}{2}(x-x_{t+1})^2 - \epsilon(x- x_{t+1}) \\
                &+ \dfrac{(x_{t+1} - x_t)^2}{4} - (x_{t+1} - x_t)\epsilon
            \end{aligned} & x\in \left(\dfrac{x_t + x_{t+1}}{2}, x_{t+1}\right].
        \end{cases}
    \end{align*}
    Without loss of generality, we can assume $x_1 = \eta$, otherwise, we simply need to perform a translation on our function $p$. Plugging in $x_t = t\eta$, we have:
    \begin{align*}
        \chi_{t, \epsilon}(x) =
        \begin{cases}
            \dfrac{1}{2}(x-t\eta)^2 - \epsilon(x - t\eta) & x\in \left(x_t, \dfrac{x_t + x_{t+1}}{2}\right];\\
            -\dfrac{1}{2}(x-(t+1)\eta)^2 - \epsilon(x-(t+1)\eta) + \dfrac{\eta^2}{4} -\eta\epsilon & x\in \left(\dfrac{x_t + x_{t+1}}{2}, x_{t+1}\right].
        \end{cases}
    \end{align*}
    Further, we can give the value $c_t$ as follows:
    \begin{equation}
    \label{Lemma1:eq1}
        c_1 = 0, \quad c_{t+1} = t\brac{\dfrac{1}{4}\eta^2 - \eta\epsilon},\quad t=1,\dots,n.
    \end{equation}
    From the definition of $p$, we have
    \begin{align*}
        p(x) \ge 
        \begin{cases}
            0, & x\in (-\infty, x_1];\\
            \min\bbrac{c_t - \dfrac{1}{2}\epsilon^2, c_{t+1}}, & x\in (x_t, x_{t+1}];\\
            c_{n+1} - \dfrac{1}{2}\epsilon^2, &  x\in (x_{n+1}, +\infty),
        \end{cases}
    \end{align*}
    which can be written as
    \begin{align*}
        \begin{aligned}
            \inf p &\ge \min_{1\le t\le n+1} c_t - \dfrac{1}{2}\epsilon^2\\
            &= \min\bbrac{0, n\brac{\dfrac{1}{4}\eta^2 - \eta\epsilon}} - \dfrac{1}{2}\epsilon^2.
        \end{aligned}
    \end{align*}
    This implies that
    \begin{align*}
        \begin{aligned}
            p(\x_1) - \inf p &\le \max \bbrac{\dfrac{1}{2}\epsilon^2, n\brac{\eta\epsilon - \dfrac{1}{4}\eta^2}+ \dfrac{1}{2}\epsilon^2}\\
            &\le \max \bbrac{\dfrac{1}{2}\epsilon^2, n\epsilon^2+ \dfrac{1}{2}\epsilon^2}\\
            &=  n\epsilon^2+ \dfrac{1}{2}\epsilon^2.
        \end{aligned}
    \end{align*}
    With~\eqref{eq:lem-ssgd-2-ass}, we have
    \begin{align*}
        n\epsilon^2+ \dfrac{1}{2}\epsilon^2 \le \dfrac{1}{4} + \dfrac{1}{8n} < 1.
    \end{align*}
    This completes the proof.
\end{proof}

\begin{lemma}\label{lemma:ssgd:2}
Fix $L,\Delta,\eta>0$, $\sigma\ge0$, an initial point $x_0\in\R$, and integers $B,T\ge1$, and write $N=BT$.
For every $\epsilon$ satisfying
\[
    0<\epsilon\le\frac{\sqrt{L\Delta}}4,
    \qquad
    N\le\frac{L\Delta(\sigma^2+4\epsilon^2)}{128\epsilon^4},
\]
there exists a lower-bounded function $p:\R\to\R$ with an $L$-Lipschitz derivative and a globally unbiased stochastic gradient oracle such that
$p(x_0)-\inf p=\Delta$.
The per-sample gradients are conditionally independent, have conditional variance at most $\sigma^2$, and the SignSGD iterates satisfy
\[
    \E\min_{0\le t<T}|p'(x_t)|\ge\epsilon.
\]
In particular, if $N\ge\max\{1,2\sigma^2/(L\Delta)\}$, there is such an instance satisfying, for a universal constant $C>0$,
\[
    \E\min_{0\le t<T}|p'(x_t)|
    \ge C\max\left\{
        \sqrt{\frac{L\Delta}{N}},
        \left(\frac{L\Delta\sigma^2}{N}\right)^{1/4}
    \right\}.
\]
\end{lemma}

\begin{proof}
Set
\[
    a=2\epsilon,\qquad
    K=\left\lfloor\frac{L\Delta}{16\epsilon^2}\right\rfloor\ge1,
    \qquad z_j=x_0+j\eta,\quad 0\le j\le K.
\]
Initially define $p(x_0)=0$ and
\[
p'(x)=
\begin{cases}
    -a, & x\le z_0,\\
    -a+L\min\{x-z_j,z_{j+1}-x\},
        & z_j\le x\le z_{j+1},\quad 0\le j<K,\\
    -a+L(x-z_K), & x\ge z_K.
\end{cases}
\]
The derivative is continuous and piecewise affine with slopes of magnitude at most $L$, and $p'(z_j)=-a$.
On each grid interval, the first quadratic piece is bounded below by $p(z_j)-a^2/(2L)$, and the second piece is concave.
Since
$p(z_{j+1})-p(z_j)=L\eta^2/4-a\eta$, it follows that
\[
\begin{aligned}
    p(x_0)-\inf p
    &\le K\left(a\eta-\frac{L\eta^2}{4}\right)_+
        +\frac{a^2}{2L}\\
    &\le\left(K+\frac12\right)\frac{a^2}{L}
    \le\frac{3\Delta}{8}.
\end{aligned}
\]
Here $u_+=\max\{u,0\}$, and
$a\eta-L\eta^2/4\le a^2/L$ holds for every $\eta>0$.

To make the gap equal to $\Delta$, choose
$R>\max\{x_0+T\eta,z_K+a/L\}$ and write
\[
    h=p'(R)>0,\qquad
    c=\sqrt{L(p(R)+\Delta)+h^2/2}>0.
\]
Beyond $R$, let $p'$ decrease from $h$ to $-c$ with slope $-L$, then increase from $-c$ to zero with slope $L$, and remain zero thereafter.
The change in function value along this tail is
\[
    \frac{h^2-c^2}{2L}-\frac{c^2}{2L}
    =-p(R)-\Delta.
\]
The new tail first increases and then decreases to $-\Delta$, so $\inf p=-\Delta$ and $p(x_0)-\inf p=\Delta$.
The derivative remains $L$-Lipschitz.
Every SignSGD query lies in $[x_0-T\eta,x_0+T\eta]$, so this modification does not change any queried derivative.

Let $\mathcal G=\{z_0,\ldots,z_K\}$ and
\[
    q=\frac{4\epsilon^2}{\sigma^2+4\epsilon^2}.
\]
For all $0\le t<T$ and $b\in[B]$, draw mutually independent
$\xi_t^b\sim\operatorname{Bernoulli}(q)$ and define
\[
    g_t^b=
    \begin{cases}
        -2\epsilon\,\xi_t^b/q, & x_t\in\mathcal G,\\
        p'(x_t), & x_t\notin\mathcal G.
    \end{cases}
\]
At every query point, the samples are conditionally independent and unbiased.
Their conditional variance equals
$4\epsilon^2(q^{-1}-1)=\sigma^2$ on $\mathcal G$ and zero elsewhere.

Define the latent batch-success variables and their total count by
\[
    Z_t=\mathbf1\left\{\sum_{b=1}^B\xi_t^b\ge1\right\},
    \qquad M=\sum_{t=0}^{T-1}Z_t.
\]
These variables are defined even after a possible exit from $\mathcal G$, and
\[
    \E M=T[1-(1-q)^B]\le Nq
    \le\frac{L\Delta}{32\epsilon^2}.
\]
Since $K+1>L\Delta/(16\epsilon^2)$, Markov's inequality gives
\[
    \Pr(M>K)\le\frac{\E M}{K+1}<\frac12.
\]
On the event $M\le K$, induction yields
$x_t=z_{\sum_{s=0}^{t-1}Z_s}$ for every $0\le t\le T$.
Indeed, at a grid point, the algorithm moves right by $\eta$ exactly when $Z_t=1$.
Thus every queried derivative has magnitude $2\epsilon$ on this event, and
\[
    \E\min_{0\le t<T}|p'(x_t)|
    \ge2\epsilon\,\Pr(M\le K)\ge\epsilon.
\]

For the final bound, take
\[
    \epsilon=
    \max\left\{
        \sqrt{\frac{L\Delta}{32N}},
        \left(\frac{L\Delta\sigma^2}{128N}\right)^{1/4}
    \right\}.
\]
This choice satisfies the required sample-budget inequality, and
$N\ge\max\{1,2\sigma^2/(L\Delta)\}$ ensures
$\epsilon\le\sqrt{L\Delta}/4$.
This proves the claim.
\end{proof}

\cref{lemma:ssgd:2} states the complexity lower bound for \cref{alg:ssgd} for a one-dimensional function. Now we can extend to a result in dimension $ d$, which will yield the bound in \cref{thm:ssgd_lower_infty}.

\begin{proof}[Proof of \cref{thm:ssgd_lower_infty}]
Write $S=\norm{\vsigma}_1$.
For $S>0$, assign
\[
    L_i=\frac{\sigma_i}{S}\Linf,\qquad
    \Delta_i=\frac{\sigma_i}{S}\Delta
\]
to each coordinate with $\sigma_i>0$.
For these coordinates,
\[
    \frac{\sigma_i^2}{L_i\Delta_i}
    =\frac{S^2}{\Linf\Delta},
\]
so \cref{eq:ssgd_lower_parameter_regime} implies the sufficient-budget condition in \cref{lemma:ssgd:2} for every coordinate.
Apply that lemma with the same $\eta,B,T$, initial point $x_{0,i}$, and parameters $(L_i,\Delta_i,\sigma_i)$.
For coordinates with $\sigma_i=0$, take $p_i\equiv0$ and a zero oracle.
Choose the oracle samples independently across coordinates and set
$f(\x)=\sum_i p_i(x_i)$.

The resulting vector samples satisfy \cref{ass:variance}.
Moreover, $\sum_iL_i=\Linf$ and
$f(\x_0)-\inf f=\sum_i\Delta_i=\Delta$, so
\cref{lem:smooth-equivalence} gives \cref{ass:infty smoothness}.
For every trajectory,
\[
    \min_{0\le t<T}\norm{\nabla f(\x_t)}_1
    \ge\sum_i\min_{0\le t<T}|p_i'(x_{t,i})|.
\]
Taking expectations and applying \cref{lemma:ssgd:2} therefore gives
\[
\begin{aligned}
    \E\min_{0\le t<T}\norm{\nabla f(\x_t)}_1
    &\ge C\sum_{i:\sigma_i>0}
        \max\left\{
            \sqrt{\frac{L_i\Delta_i}{N}},
            \left(\frac{L_i\Delta_i\sigma_i^2}{N}\right)^{1/4}
        \right\}\\
    &\ge\frac C2\left[
        \sqrt{\frac{\Linf\Delta}{N}}
        +\left(\frac{S^2\Linf\Delta}{N}\right)^{1/4}
    \right].
\end{aligned}
\]
If $S=0$, apply \cref{lemma:ssgd:2} with
$(L,\Delta,\sigma)=(\Linf,\Delta,0)$ to one coordinate and take all remaining coordinates constant.
This proves \eqref{eq:lower_signsgd}.
\end{proof}

\section{Proof of Theorem~\ref{thm:sgd_lower_infty}\label{sec:proof-sgd-lower-bound}}

We first establish a one-dimensional lower bound, motivated by \citet[Theorem~4.1]{jiang2024convergence}, and then extend it to $d$ dimensions.

Let $C>0$ be a sufficiently large universal constant. The parameter range in \cref{thm:sgd_lower_infty} is
\begin{equation}\label{eq:sgd_lower_parameter_regime}
B,T\in\mathbb N_{\ge1},\qquad N=BT,\qquad T\ge Cd^3,\qquad
N\ge C\frac{d^5\norm{\vsigma}_2^2}{\Linf\Delta}.
\end{equation}
The constant step size $\eta>0$ is arbitrary. At each iteration, SGD uses $B$ conditionally independent per-sample gradients and updates
\[
\x_{t+1}=\x_t-\frac{\eta}{B}\sum_{b=1}^B\g_t^b,\qquad 0\le t<T.
\]
All constants in the proof are universal.

\begin{lemma}\label{lem:sgd_batch_one_dim}
For any $L,\delta>0$, $\sigma\ge0$, $\eta>0$, and integers $B,T\ge1$, there exist a function $p:\R\to\R$, an initial point $x_0$, and a stochastic gradient oracle such that:
\begin{enumerate}
\setlength{\itemsep}{2pt}\setlength{\parsep}{0pt}\setlength{\parskip}{0pt}
    \item[(i)] $p$ has an $L$-Lipschitz derivative, is lower bounded, and satisfies
    $p(x_0)-\inf p\le\delta$;
    \item[(ii)] at every query point $x\in\R$, the oracle samples are conditionally independent and satisfy
    \[
    \E[g_t^b\mid x_t=x]=p'(x),\qquad
    \E[(g_t^b-p'(x))^2\mid x_t=x]\le\sigma^2,\qquad b\in[B];
    \]
    \item[(iii)] the mini-batch SGD iterates
    $x_{t+1}=x_t-\frac{\eta}{B}\sum_{b=1}^B g_t^b$
    satisfy
    \[
    \E\min_{0\le t<T}|p'(x_t)|\ge\sqrt{2L\delta}
    \qquad\text{if }\eta\ge2/L.
    \]
    If $\eta<2/L$, they satisfy
    \begin{equation}\label{eq:sgd_batch_1d}
    \E\min_{0\le t<T}|p'(x_t)|
    \ge\max\left\{
    \frac12\sqrt{\frac{\delta}{2\eta T+1/(2L)}},
    \min\left(\sigma\sqrt{\frac{L\eta}{2B}},\sqrt{2L\delta}\right)
    \right\}.
    \end{equation}
\end{enumerate}
\end{lemma}

\begin{proof}
For $a>0$, we define
\[
p(x)=
\begin{cases}
    \dfrac{Lx^2}{2}, & |x|\le a/L,\\[1ex]
    a|x|-\dfrac{a^2}{2L}, & |x|>a/L.
\end{cases}
\]
Then $p'(x)=\operatorname{clip}(Lx,-a,a)$, so $p'$ is $L$-Lipschitz and $\inf p=0$.
For $\eta\ge2/L$, we choose $a=\sqrt{2L\delta}$, set $x_0=a/L$, and use the exact oracle $g_t^b=p'(x_t)$.
The iterates satisfy
\[
x_{2t}=\frac aL,\qquad
x_{2t+1}=\frac aL-\eta a\le-\frac aL.
\]
Thus $|p'(x_t)|=a$ at every iteration, and
$p(x_0)-\inf p=a^2/(2L)=\delta$.

Now suppose $\eta<2/L$.
For $\sigma>0$, we first choose
\[
a=\min\left\{\sigma\sqrt{\frac{L\eta}{2B}},\sqrt{2L\delta}\right\},
\qquad
\alpha=\frac{a^2}{\sigma^2+a^2},
\qquad x_0=\frac aL.
\]
We use the same function $p$ and define the oracle by
\[
g_t^b=\frac{\xi_t^b}{\alpha}p'(x_t),
\qquad \xi_t^b\sim\operatorname{Bernoulli}(\alpha),
\]
where the variables $\xi_t^b$, $0\le t<T$, $b\in[B]$, are independent.
By construction,
\begin{align*}
\E[g_t^b\mid x_t=x]
&=p'(x),\\
\E[(g_t^b-p'(x))^2\mid x_t=x]
&=(\alpha^{-1}-1)p'(x)^2
\le(\alpha^{-1}-1)a^2=\sigma^2.
\end{align*}
Let $K_t=\sum_{b=1}^B\xi_t^b$.
Whenever $|p'(x_t)|=a$, the update is
\[
x_{t+1}
=x_t-hK_t\sign{p'(x_t)},
\qquad
h=\frac{\eta a}{B\alpha}
\ge\frac{\eta\sigma^2}{Ba}\ge\frac{2a}{L}.
\]
Starting from $x_0=a/L$, induction therefore gives
$x_t\in a/L+h\mathbb Z$ for every $t$.
This lattice has no point in $(-a/L,a/L)$, so $|p'(x_t)|=a$.
Consequently,
\[
p(x_0)-\inf p=\frac{a^2}{2L}\le\delta,
\qquad
\E\min_{0\le t<T}|p'(x_t)|
=a=\min\left\{\sigma\sqrt{\frac{L\eta}{2B}},\sqrt{2L\delta}\right\}.
\]

For the other term in~\eqref{eq:sgd_batch_1d}, we set $x_0=0$ and choose
\[
e=\frac12\sqrt{\frac{\delta}{2\eta T+1/(2L)}},
\qquad
\alpha=\frac{4e^2}{\sigma^2+4e^2},
\qquad
s=\frac{2\eta e}{B\alpha},
\qquad
M=\lfloor2BT\alpha\rfloor.
\]
Let $q_j=js$ for $0\le j\le M$.
We now define $p(x)=\int_0^x p'(u)\,du$, where
\[
p'(x)=
\begin{cases}
    -2e+Lx,
    & x\le0,\\
    -2e+L\min\{x-q_j,q_{j+1}-x\},
    & q_j\le x\le q_{j+1},\quad 0\le j<M,\\
    -2e+L(x-q_M),
    & q_M\le x\le q_M+2e/L,\\
    0,
    & x\ge q_M+2e/L.
\end{cases}
\]
The pieces agree at their endpoints and have slopes of magnitude at most $L$.
Thus $p'$ is globally $L$-Lipschitz, and $p'(q_j)=-2e$ for every grid point.
Integrating the derivative gives
\begin{align*}
p(q_j)-p(q_{j+1})
&=2es-\frac{Ls^2}{4}
\le2es=\frac{\eta(\sigma^2+4e^2)}B,\\
\inf p
&\ge\min_{0\le j\le M}p(q_j)-\frac{2e^2}{L}.
\end{align*}
Here $p(x)\ge p(0)$ for $x\le0$.
On each grid interval, the function falls at most $2e^2/L$ below the smaller endpoint value, and the final tail has minimum $p(q_M)-2e^2/L$.
It follows that $p$ is lower bounded and
\begin{align*}
p(0)-\inf p
&\le M\frac{\eta(\sigma^2+4e^2)}B+\frac{2e^2}{L}\\
&\le8e^2\eta T+\frac{2e^2}{L}
=\delta.
\end{align*}

Let $\xi_t^b$, $0\le t<T$, $b\in[B]$, be independent $\operatorname{Bernoulli}(\alpha)$ variables and define the oracle by
\[
g_t^b=
\begin{cases}
    -2e\xi_t^b/\alpha, & x_t\in\{q_0,\ldots,q_M\},\\
    p'(x_t), & \text{otherwise}.
\end{cases}
\]
For every $x\in\R$, this gives
\begin{align*}
\E[g_t^b\mid x_t=x]
&=p'(x),\\
\E[(g_t^b-p'(x))^2\mid x_t=x]
&=
\begin{cases}
    4e^2(\alpha^{-1}-1)=\sigma^2,
    & x\in\{q_0,\ldots,q_M\},\\
    0, & \text{otherwise}.
\end{cases}
\end{align*}
Set $K_t=\sum_{b=1}^B\xi_t^b$ and $S=\sum_{t=0}^{T-1}K_t$.
Since $S$ is the sum of all $BT$ independent Bernoulli variables, we have
\[
\E S=BT\alpha,\qquad
\Pr(S>M)=\Pr(S\ge M+1)
\le\frac{BT\alpha}{M+1}<\frac12.
\]
On the event $S\le M$, induction gives
\[
x_t=q_{\sum_{r=0}^{t-1}K_r},
\qquad
x_{t+1}=q_{\sum_{r=0}^{t}K_r},
\qquad 0\le t<T.
\]
All the indices are at most $M$ on this event, so every queried gradient has magnitude $2e$.
Therefore,
\[
\E\min_{0\le t<T}|p'(x_t)|
\ge2e\Pr(S\le M)\ge e
=\frac12\sqrt{\frac{\delta}{2\eta T+1/(2L)}}.
\]
Choosing the function that gives the larger of the two lower bounds proves \eqref{eq:sgd_batch_1d}.
\end{proof}

\begin{proof}[Proof of \cref{thm:sgd_lower_infty}]
We first fix $\L=(L_1,\ldots,L_d)$ with $L_i>0$ and $\sum_iL_i=\Linf$, and set $\delta=\Delta/d$.
For each coordinate, we apply \cref{lem:sgd_batch_one_dim} with parameters $(L_i,\delta,\sigma_i)$ and the same $\eta,B,T$.
We translate the functions so that their initial points agree with the prescribed coordinates of $\x_0$, and choose the samples independently across coordinates.
Then $f(\x)=\sum_i p_i(x_i)$ satisfies \cref{ass:separable smoothness}, and hence \cref{ass:infty smoothness} by \cref{lem:smooth-equivalence}.
Its initial gap is at most $\Delta$, and the vector oracle is globally unbiased with the required coordinatewise variance.
For every trajectory, we have
\[
\min_{0\le t<T}\norm{\nabla f(\x_t)}_1
=\min_{0\le t<T}\sum_i|p_i'(x_{t,i})|
\ge\sum_i\min_{0\le t<T}|p_i'(x_{t,i})|.
\]
Write $L_{\max}=\max_iL_i$, $L_{\min}=\min_iL_i$, and
$C_\sigma=\sum_i\sigma_i\sqrt{L_i}$.
If $\eta\ge2/L_{\max}$, the coordinate with largest curvature gives a lower bound of
$\sqrt{2L_{\max}\Delta/d}$.
If $\eta<2/L_{\max}$, all coordinates satisfy \eqref{eq:sgd_batch_1d}.
Summing these bounds and using $\max\{a,b\}\ge(a+b)/2$, we get, for a universal $c>0$,
\begin{equation}\label{eq:sgd_batch_coordinate}
\E\min_{0\le t<T}\norm{\nabla f(\x_t)}_1\ge c\left[
\sqrt{\frac{d\Delta}{\eta T+1/L_{\min}}}
+\sum_i\min\left\{\sigma_i\sqrt{\frac{L_i\eta}{B}},\sqrt{\frac{L_i\Delta}{d}}\right\}\right].
\end{equation}
When $\eta T\le1/L_{\min}$, the first term is at least
$c\sqrt{d\Delta L_{\min}}$.
When a stochastic term is capped, it gives at least
$c\sqrt{L_{\min}\Delta/d}$.
Otherwise, the right side of \eqref{eq:sgd_batch_coordinate} is at least
\[
c\left[\sqrt{\frac{d\Delta}{\eta T}}+C_\sigma\sqrt{\frac\eta B}\right].
\]
Since $\eta<2/L_{\max}$, the first summand is at least a constant multiple of
$\sqrt{dL_{\max}\Delta/T}$.
By the arithmetic-geometric mean inequality, the sum is also at least a constant multiple of
$(d\Delta C_\sigma^2/(BT))^{1/4}$.
In each of the other cases, namely $\eta\ge2/L_{\max}$, $\eta T\le1/L_{\min}$, or a capped stochastic term, the lower bound is at least
$c\sqrt{L_{\min}\Delta/d}$.
These lower bounds imply the two required rates provided that
\begin{equation}\label{eq:sgd_batch_regime}
T\ge C\frac{d^2L_{\max}}{L_{\min}},\qquad
BT\ge C\frac{d^3 C_\sigma^2}{L_{\min}^2\Delta}.
\end{equation}
Indeed, these conditions give
\begin{align*}
\sqrt{\frac{dL_{\max}\Delta}{T}}
&\le C^{-1/2}\sqrt{\frac{L_{\min}\Delta}{d}},\\
\left(\frac{d\Delta C_\sigma^2}{BT}\right)^{1/4}
&\le C^{-1/4}\sqrt{\frac{L_{\min}\Delta}{d}}.
\end{align*}
Combining these cases, under \eqref{eq:sgd_batch_regime} we obtain
\begin{equation}\label{eq:lower-sgd-coordinate}
\E\min_{0\le t<T}\norm{\nabla f(\x_t)}_1
=\Omega\left(\sqrt{\frac{dL_{\max}\Delta}{T}}
+\left(\frac{d\Delta C_\sigma^2}{BT}\right)^{1/4}\right).
\end{equation}

For the deterministic term, we choose
$L_1=(d+1)\Linf/(2d)$ and $L_i=\Linf/(2d)$ for $i>1$.
Then $L_{\max}=\Omega(\Linf)$, $L_{\min}\ge\Linf/(2d)$, and
$C_\sigma^2\le2\Linf\norm{\vsigma}_2^2$.
For $\vsigma\ne0$, we construct a second instance with
\[
L_i=\frac{\Linf}{2}\left(\frac{\sigma_i^2}{\norm{\vsigma}_2^2}+\frac1d\right)>0.
\]
This choice gives $L_{\min}\ge\Linf/(2d)$ and
$\sqrt{\Linf/2}\norm{\vsigma}_2\le C_\sigma\le\sqrt{2\Linf}\norm{\vsigma}_2$.
For both choices, \eqref{eq:sgd_lower_parameter_regime} implies \eqref{eq:sgd_batch_regime} after enlarging the same universal constant $C$.
Applying \eqref{eq:lower-sgd-coordinate} to these two instances and using
$\max\{a,b\}\ge(a+b)/2$, we can choose one of them to obtain
\begin{equation}\label{eq:lower_sgd_batch_strong}
\expect{\min_{0\le t<T}\norm{\nabla f(\x_t)}_1}
=\Omega\brac{\sqrt{\frac{dB\Linf\Delta}{N}}+
\brac{\frac{d\norm{\vsigma}_2^2\Linf\Delta}{N}}^{\frac14}}.
\end{equation}
Since $B\ge1$, this implies \eqref{eq:lower_sgd_infty}.
If $\vsigma=0$, we use only the first instance.

It remains to make the initial gap equal to $\Delta$.
Each one-dimensional construction gives a gap at most its assigned budget.
Since $B$ and $T$ are finite and the oracle has only finitely many Bernoulli outcomes, the set of all possible queried points is finite.
Outside this set, we extend a tail by a sufficiently wide smooth well so that the gap equals the assigned budget.
This preserves the $L_i$-smoothness and all queried gradients.
For the stochastic clipped-quadratic construction, we also preserve $|p_i'|\le a$, so the global variance bound is unchanged.
This completes the proof.
\end{proof}
%MUON--------------------------------------------%

\section{Analysis for Muon}
\subsection{Upper Bound for Muon\label{sec:proof-muon-upperbound}}

\begin{proof}[Proof of Theorem~\ref{thm:muon_upper_infty}]
    Set the hyperparameters to
    \begin{equation}\label{eq:hyp_muon}
        \eta = \sqrt{\dfrac{2\Delta}{L_*T}},\quad B = \left\lceil\max \bbrac{1, \dfrac{\norm{\vSigma}_*^2}{\Delta L_*}T}\right\rceil.
    \end{equation}
    The batch size in \eqref{eq:hyp_muon} grows with $T$ when $\norm{\vSigma}_*>0$. At fixed $B$, \eqref{eq:muon_convegence_wo_eta} retains the term $2\norm{\vSigma}_*/\sqrt{B}$, so this theorem is not a constant-batch guarantee for arbitrarily small stationarity.

    Under \cref{ass:Spectral norm smooth}, we have:
    \begin{align*}
        F(\WB_{t+1}) &\le F(\WB_t) + \inner{\nabla F(\WB_t)}{\WB_{t+1} - \WB_t} + \dfrac{L_*}{2}\norm{\WB_{t+1} - \WB_t}_{\op}^2\\
        &= F(\WB_t) - \inner{\nabla F(\WB_t)}{\eta\msign{\Gb_t}} + \dfrac{L_*}{2}\norm{\eta\msign{\Gb_t}}_{\op}^2\\
        &\le F(\WB_t) - \eta\inner{\Gb_t}{\msign{\Gb_t}} - \eta\inner{\nabla F(\WB_t) - \Gb_t}{\msign{\Gb_t}} + \dfrac{L_*\eta^2}{2}\\
        &\le F(\WB_t) - \eta \norm{\Gb_t}_* +\eta \norm{\nabla F(\WB_t) - \Gb_t}_*\norm{\msign{\Gb_t}}_{\op} + \dfrac{L_*\eta^2}{2}\\
        &\le F(\WB_t) - \eta \norm{\nabla F(\WB_t)}_* +2\eta \norm{\nabla F(\WB_t) - \Gb_t}_* + \dfrac{L_*\eta^2}{2},
    \end{align*}
    where we use $\norm{\msign{\Gb_t}}_{\op}\le1$ and $\abs{\inner{A}{B}} \le \norm{A}_{*}\norm{B}_{\op}$. Summing over $t = 0, \dots, T-1$, we get
    \begin{equation}
        F(\WB_T) - F(\WB_0) \le -\eta \sum_{t=0}^{T-1} \norm{\nabla F(\WB_t)}_*  + 2\eta \sum_{t=0}^{T-1} \norm{\nabla F(\WB_t) - \Gb_t}_* + \dfrac{L_* \eta^2 T}{2}.
    \end{equation}
    Under \cref{ass:Matrix-bounded-below}, $F(\WB_0) - F(\WB_T)\le F(\WB_0) - \inf F \le \Delta$, take expectation on both sides, and we get
    \begin{equation}\label{eq:muon-convergence}
        \dfrac{1}{T}\sum_{t=0}^{T-1} \expect{\norm{\nabla F(\WB_t)}_*} \le \dfrac{\Delta}{\eta T}  + \dfrac{2}{T} \sum_{t=0}^{T-1} \expect{\norm{\nabla F(\WB_t) - \Gb_t}_*} + \dfrac{L_*\eta}{2}.
    \end{equation}
    Define $\NB_t = \nabla F(\WB_t) - \Gb_t$. Consider Batch size $B$, we denote the gradient on each sample is $\Gb_t^b$, $1 \le b \le B$, and the noise on each sample as $\NB_t^b = \nabla F(\WB_t) - \Gb_t^b$. We have the following lemma:
    \begin{lemma}\label{lem:variance-reduction-matrix}
        (\citet[Lemma 9]{an2025asgo}) Under \cref{ass:Matrix-variance}, with batch size $B$, we have
        \begin{equation}
            \expect{\NB_t\NB_t^\top} \preceq \dfrac{\vSigma^2}{B}.
        \end{equation}
    \end{lemma}
    Under the fact that the map $X \mapsto \Tr{X^{\frac{1}{2}}}$ is concave on the positive semi-definite cone, we can apply Jensen's inequality to get
    \begin{equation}\label{eq:variance-reduction-matrix}
        \expect{\norm{\NB_t}_*} = \expect{\Tr{\sqrt{\NB_t\NB_t^\top}}} \le \Tr{\sqrt{\expect{\NB_t\NB_t^\top}}} \le \dfrac{\norm{\vSigma}_*}{\sqrt{B}}.
    \end{equation}
    Plugging~\eqref{eq:variance-reduction-matrix} back to~\eqref{eq:muon-convergence} and set $\eta = \sqrt{\nicefrac{2\Delta}{L_* T}}$, we get
    \begin{equation}\label{eq:muon_convegence_wo_eta}
        \dfrac{1}{T}\sum_{t=0}^{T-1} \expect{\norm{\nabla F(\WB_t)}_*} 
        \le \sqrt{\dfrac{2\Delta L_*}{T}} + 2\dfrac{\norm{\vSigma}_*}{\sqrt{B}}.
    \end{equation}
    If $\norm{\vSigma}_*^2 T \le \Delta L_*$, by setting $B = 1$, we have
    \begin{align*}
        \expect{\dfrac{1}{T}\sum_{t=0}^{T-1}\norm{\nabla F(\WB_t)}_*}
            \le \sqrt{\dfrac{2\Delta L_*}{T}} + 2\norm{\vSigma}_* \le (2+\sqrt{2}) \sqrt{\dfrac{\Delta L_*}{N}}
    \end{align*}
    Otherwise, we set $B = \left\lceil\nicefrac{\norm{\vSigma}_*^2}{\Delta L_*}T\right\rceil \le \nicefrac{2\norm{\vSigma}_*^2}{\Delta L_*}T$, which implies
    \begin{equation}\label{eq:muon_deter}
        \sqrt{\dfrac{2\Delta L_*}{T}} \le \brac{\dfrac{8\norm{\vSigma}_*^2 \Delta L_*}{BT}}^{\frac{1}{4}}.
    \end{equation}
    Bringing~\eqref{eq:muon_deter} back to~\eqref{eq:muon_convegence_wo_eta}
    \begin{align*}
        \expect{\dfrac{1}{T}\sum_{t=0}^{T-1}\norm{\nabla F(\WB_t)}_*}
        &\le 2\brac{\dfrac{\norm{\vSigma}^2_*\Delta L_*}{BT}}^{\frac{1}{4}} + \brac{\dfrac{8\norm{\vSigma}_*^2 \Delta L_*}{BT}}^{\frac{1}{4}}\\
        &= \bigO{\brac{\dfrac{\norm{\vSigma}_*^2 \Delta L_*}{BT}}^{\frac{1}{4}}}.
    \end{align*}
    Putting the above two cases together, we have
    \begin{align*}
        \expect{\dfrac{1}{T}\sum_{t=0}^{T-1}\norm{\nabla F(\WB_t)}_*} = \bigO{{\sqrt{\dfrac{\Delta L_*}{N}}} + \brac{\dfrac{\norm{\vSigma}_*^2\Delta L_*}{N}}^{\frac{1}{4}}}.
    \end{align*}
\end{proof}

\subsection{Lower Bound for Muon\label{sec:proof-muon-lowerbound}}

The parameter range in \cref{thm:muon_lower_infty} is
\begin{equation}\label{eq:muon_lower_parameter_regime}
    B,T\in\mathbb N_{\ge1},\qquad N=BT,\qquad
    N\ge\max\left\{1,\frac{2\norm{\vSigma}_*^2}{L_*\Delta}\right\}.
\end{equation}

Assume the target weight matrix is $\WB \in \R^{m \times n}$. Without loss of generality, we assume $m \le n$. We introduce a projection matrix $\PB = [\Ib_m, \mathbf{0}_{m \times (n-m)}] \in \R^{m \times n}$ which pads an $m \times m$ matrix with zeros to match the $m \times n$ dimension. Let $\vSigma \in \R^{m \times m}$ be the target noise-scale matrix, which can be diagonalized via an orthogonal matrix $\QB \in \R^{m \times m}$ as $\vSigma = \QB^\top \diag(\vsigma) \QB$.

We construct our worst-case matrix objective function by extracting the diagonal elements of the transformed matrix $\QB\WB\PB^\top \in \R^{m \times m}$:
\begin{equation}
    F(\WB) = \sum_{i=1}^m f_i\big((\QB\WB\PB^\top)_{ii}\big),
\end{equation}
where $f(\x) = \sum_{i=1}^m f_i(\x_i)$ is the hard 1D vector instance constructed in \cref{thm:ssgd_lower_infty}.
Apply that theorem with $\Linf=L_*$, the eigenvalue vector $\vsigma$ of $\vSigma$, and $x_{0,i}=(\QB\WB_0\PB^\top)_{ii}$.
Since $\norm{\vsigma}_1=\norm{\vSigma}_*$, \eqref{eq:muon_lower_parameter_regime} is exactly the vector budget condition.
The diagonal map is onto, with right inverse $\x\mapsto\QB^\top\diag(\x)\PB$, so $\inf F=\inf f$ and $F(\WB_0)-\inf F=\Delta$.

First, we provide two lemmas to establish the strict geometric equivalence between the vector domain and the matrix domain.

\begin{lemma}\label{lem:variance_equivalance}
    Assume that the separable vector function $f(\x) : \R^m \to \R$ satisfies \cref{ass:variance} with bounded variance $\vsigma^2$. Then the constructed matrix function $F(\WB) : \R^{m \times n} \to \R$ satisfies \cref{ass:Matrix-variance} with matrix covariance bound $\vSigma^2 = \QB^\top\diag(\vsigma^2)\QB$.
\end{lemma}

\begin{proof}
    By the chain rule, the exact gradient of $F(\WB)$ is given by:
    \begin{equation}\label{eq:matrix-true-grad}
        \nabla F(\WB) = \QB^\top \diag\big(\nabla f(\x)\big) \PB,
    \end{equation}
    where $\x \in \R^m$ with $\x_i = (\QB\WB\PB^\top)_{ii}$.
    At every matrix query $\WB_t$, draw fresh vector samples $\g_t^b$, $b\in[B]$, conditionally independently given $\mathcal F_{t-1}$, from the global vector oracle at $x_{t,i}=(\QB\WB_t\PB^\top)_{ii}$.
    Lift each sample linearly as $\Gb_t^b=\QB^\top\diag(\g_t^b)\PB$; write $\g=\g_t^b$ and $\Gb=\Gb_t^b$ for one such sample:
    \begin{equation}\label{eq:matrix-stoc-grad}
        \Gb = \QB^\top \diag(\g) \PB.
    \end{equation}
    The estimation error matrix is $\Gb - \nabla F(\WB) = \QB^\top \diag(\g - \nabla f(\x)) \PB$. Computing the uncentered covariance, we have:
    \begin{equation*}
    \begin{aligned}
        &\expect{(\Gb - \nabla F(\WB))(\Gb - \nabla F(\WB))^\top \mid \mathcal F_{t-1}} \\
        &\quad= \expect{\QB^\top \diag(\g - \nabla f(\x)) \PB \PB^\top \diag(\g - \nabla f(\x)) \QB \mid \mathcal F_{t-1}} \\
        &\quad= \QB^\top \expect{\diag\big((g_i - \nabla f_i(\x_i))^2\big) \mid \mathcal F_{t-1}} \QB.
    \end{aligned}
    \end{equation*}
    Here we used $\PB\PB^\top=\Ib_m$. Applying the coordinate-wise conditional variance bound in \cref{ass:variance} gives:
    \begin{equation*}
        \expect{(\Gb - \nabla F(\WB))(\Gb - \nabla F(\WB))^\top \mid \mathcal F_{t-1}} \preceq \QB^\top\diag(\vsigma^2)\QB = \vSigma^2.
    \end{equation*}
    Conditional unbiasedness, $\E[\Gb_t^b\mid\mathcal F_{t-1}]=\nabla F(\WB_t)$, and conditional independence across batch samples are inherited from the vector oracle under this deterministic linear lifting.
    This completes the proof.
\end{proof}

\begin{lemma}\label{lem:vector-matrix-smoothness_equivalance}
    Assume that $f(\x) : \R^m \to \R$ satisfies \cref{ass:infty smoothness} with $\Linf$-smoothness parameter $L$. Then the constructed matrix function $F(\WB): \R^{m \times n} \to \R$ satisfies \cref{ass:Spectral norm smooth} with spectral norm smoothness $L_* = L$.
\end{lemma}

\begin{proof}
    Let $\XB, \YB \in \R^{m \times n}$ be two arbitrary matrices, and define $\Delta = \YB - \XB$. We map these matrices to vectors $\x, \y \in \R^m$ by defining $\x_i = (\QB\XB\PB^\top)_{ii}$ and $\y_i = (\QB\YB\PB^\top)_{ii}$. Let their difference be $\boldsymbol{\delta} = \y - \x$, where $\delta_i = (\QB\Delta\PB^\top)_{ii}$. By definition, $F(\XB) = f(\x)$ and $F(\YB) = f(\y)$.
    
    First, we bridge the Frobenius inner product of the matrices with the Euclidean inner product of the vectors. Let $\SB = \diag(\nabla f(\x))$. Using~\eqref{eq:matrix-true-grad} and the cyclic property of the trace operator, we have:
    \begin{equation*}
    \begin{aligned}
        \inner{\nabla F(\XB)}{\Delta}_F &= \text{tr}\big((\QB^\top \SB \PB)^\top \Delta\big) \\
        &= \text{tr}\big(\PB^\top \SB^\top \QB \Delta\big) \\
        &= \text{tr}\big(\SB^\top (\QB \Delta \PB^\top)\big).
    \end{aligned}
    \end{equation*}
    Let $\MB = \QB \Delta \PB^\top \in \R^{m \times m}$. Since $\SB$ is diagonal, the trace $\text{tr}(\SB^\top \MB)$ perfectly reduces to the element-wise sum of their main diagonals:
    \begin{equation*}
        \text{tr}\big(\SB^\top \MB\big) = \sum_{i=1}^m \SB_{ii} \MB_{ii} = \sum_{i=1}^m \nabla f_i(\x_i) \delta_i = \inner{\nabla f(\x)}{\y - \x}.
    \end{equation*}
    Therefore, the Taylor expansion residuals in both domains are strictly identical:
    \begin{equation}\label{eq:residual_equiv}
        F(\YB) - \left(F(\XB) + \inner{\nabla F(\XB)}{\YB-\XB}_F\right) = f(\y) - \left(f(\x) + \inner{\nabla f(\x)}{\y-\x}\right).
    \end{equation}
    Taking the absolute value and applying the $\Linf$-smoothness of $f(\x)$, we obtain:
    \begin{equation}\label{eq:smoothness_bound}
        \abs{F(\YB) - F(\XB) - \inner{\nabla F(\XB)}{\YB-\XB}_F} \le \dfrac{L}{2} \norm{\y - \x}^2_\infty.
    \end{equation}
    To bound the vector $\ell_\infty$-norm using the matrix spectral norm, we examine the maximum component of $\boldsymbol{\delta}$. For any $i \in \{1, \dots, m\}$:
    \begin{equation*}
        \abs{\y_i - \x_i} = \abs{(\QB\Delta\PB^\top)_{ii}} = \abs{\e_i^\top \QB \Delta \PB^\top \e_i} = \abs{\q_i^\top \Delta \tilde{\e}_i},
    \end{equation*}
    where $\e_i \in \R^m$ is the standard basis vector, $\q_i = \QB^\top \e_i \in \R^m$, and $\tilde{\e}_i = \PB^\top \e_i \in \R^n$. Since $\QB$ is orthogonal and $\PB$ is a standard projection, both vectors preserve their unit Euclidean lengths: $\norm{\q_i}_2 = 1$ and $\norm{\tilde{\e}_i}_2 = 1$. By the definition of the induced matrix 2-norm (spectral norm), we have:
    \begin{equation*}
        \abs{\q_i^\top \Delta \tilde{\e}_i} \le \norm{\q_i}_2 \norm{\Delta}_{\op} \norm{\tilde{\e}_i}_2 = \norm{\YB - \XB}_{\op}.
    \end{equation*}
    Taking the maximum over all $i$ yields $\norm{\y - \x}_\infty \le \norm{\YB - \XB}_{\op}$. Substituting this into~\eqref{eq:smoothness_bound} establishes that $F(\WB)$ is $L_*$-spectral norm smooth with $L_* = L$.
\end{proof}

\begin{proof}[Proof of \cref{thm:muon_lower_infty}]
    We now demonstrate that under this construction, the optimization trajectory of Muon in the matrix space exactly replicates that of SignSGD in the vector space.
    
    Recall the constructed stochastic matrix gradient is $\Gb_t = \QB^\top \diag(\g_t) \PB$. The Muon optimizer updates the weight matrix $\WB$ by moving in the direction of the orthogonalized gradient, computed via the thin SVD: $\Gb_t = \U \bSigma \V^\top$, where the update direction is $\msign{\Gb_t} = \U \V^\top$.
    
    We can explicitly write the SVD of our constructed $\Gb_t$ by restricting to the active coordinates. Let $S = \{i: g_{t,i} \neq 0\}$ and let $\EB_S \in \R^{m \times \abs{S}}$ be the coordinate selector whose columns are the standard basis vectors indexed by $S$. Define
    \begin{equation*}
        \U = \QB^\top \EB_S \diag\big(\sign{\g_{t,S}}\big), \quad
        \bSigma = \diag\big(\abs{\g_{t,S}}\big), \quad
        \V = \PB^\top \EB_S,
    \end{equation*}
    where $\g_{t,S}$ denotes the subvector of $\g_t$ indexed by $S$. Since the diagonal entries of $\bSigma$ are strictly positive, and since $\QB^\top$ is orthogonal while $\PB \PB^\top = \Ib_m$, both $\U$ and $\V$ have orthonormal columns. Ordering the columns by $\abs{\g_{t,S}}$ yields the thin SVD of $\Gb_t$, and this ordering does not affect $\U \V^\top$. We have
    \begin{equation*}
        \U \bSigma \V^\top = \QB^\top \EB_S \diag(\sign{\g_{t,S}}) \diag(\abs{\g_{t,S}}) \EB_S^\top \PB = \QB^\top \diag(\g_t) \PB = \Gb_t,
    \end{equation*}
    so the orthogonal factor is
    \begin{equation*}
        \msign{\Gb_t} = \U \V^\top = \QB^\top \EB_S \diag\big(\sign{\g_{t,S}}\big) \EB_S^\top \PB = \QB^\top \diag(\sign{\g_t}) \PB.
    \end{equation*}
    When $S$ is empty, $\Gb_t = \mathbf{0}$ and $\msign{\Gb_t} = \mathbf{0}$.
    The Muon update rule $\WB_{t+1} = \WB_t - \eta \, \msign{\Gb_t}$ then translates to:
    \begin{equation*}
        \WB_{t+1} = \WB_t - \eta \QB^\top \diag(\sign{\g_t}) \PB.
    \end{equation*}
    Left-multiplying by $\QB$ and right-multiplying by $\PB^\top$, we obtain:
    \begin{equation*}
        \QB\WB_{t+1}\PB^\top = \QB\WB_t\PB^\top - \eta \diag(\sign{\g_t}).
    \end{equation*}
    By isolating the diagonal entries $\x_{t,i} = (\QB\WB_t\PB^\top)_{ii}$, we recover the exact SignSGD coordinate-wise update:
    \begin{equation*}
        \x_{t+1} = \x_t - \eta \sign{\g_t}.
    \end{equation*}
    This establishes a strict dynamic isomorphism: the sequence of the extracted diagonal entries perfectly replicates the trajectory of SignSGD running on the vector function $f(\x)$.
    
    By \cref{lem:variance_equivalance} and \cref{lem:vector-matrix-smoothness_equivalance}, $F(\WB)$ satisfies \cref{ass:Matrix-variance} with noise-scale matrix $\vSigma = \QB^\top \diag(\vsigma) \QB$ and \cref{ass:Spectral norm smooth} with $L_* = \Linf$. Furthermore, the trace norm (nuclear norm) of the noise-scale matrix strictly aligns: $\norm{\vSigma}_* = \text{tr}(\vSigma) = \norm{\vsigma}_1$.
    
    Finally, observe the true gradient $\nabla F(\WB_t) = \QB^\top \diag(\nabla f(\x_t)) \PB$. Because orthogonal transformations preserve singular values, the singular values of $\nabla F(\WB_t)$ are exactly the absolute values of the components of $\nabla f(\x_t)$. Thus, the nuclear norm of the matrix gradient equals the $\ell_1$-norm of the vector gradient:
    \begin{equation*}
        \norm{\nabla F(\WB_t)}_* = \norm{\nabla f(\x_t)}_1.
    \end{equation*}
    Because the optimization trajectories are identical and all bounding metrics perfectly correspond, we can directly invoke the vector lower bound from \cref{thm:ssgd_lower_infty} under \eqref{eq:muon_lower_parameter_regime} to establish the matrix lower bound for Muon:
    \begin{equation*}
        \expect{\min_{0\le t<T} \norm{\nabla F(\WB_t)}_*} = \Omega\brac{\sqrt{\dfrac{L_*\Delta}{N}} + \brac{\dfrac{\norm{\vSigma}_*^2 L_*\Delta}{N}}^{\frac{1}{4}}  }.
    \end{equation*}
    This completes the proof.
\end{proof}

\section{Experimental Details}
\label{app:experimental_details}

This section gives the numerical-example figures, real-data performance configurations, and coordinate-noise measurement protocol.

\begingroup
\setlength{\textfloatsep}{20pt plus 8pt minus 4pt}
\setlength{\floatsep}{12pt plus 6pt minus 2pt}
\subsection{CNN and LLM Performance Comparisons}
\label{app:performance-configurations}
All experiments are conducted in PyTorch on a single NVIDIA Pro 6000 GPU with 96 GB of memory. 

For ResNet-20 on CIFAR-10, we follow the experimental setup of \citet{bernstein2018signsgd}. We present the omitted result details in ~\cref{fig:sign_vs_sgd_cnn}.

For nanoGPT and the 130M-parameter LLaMA on C4, we use the codebase of \citet{semenov2025benchmarking} at \url{https://github.com/epfml/llm-optimizer-benchmark/tree/main/scripts}. All three methods use 20000 updates, global batch size 256, and sequence length 512, which gives approximately 2.6 billion training tokens, roughly $1\times$ Chinchilla-optimal token budget for these models~\citep{hoffmann2022chinchilla}. The best learning rates of SignSGD and SGD are found via grid search in $\{10^{-2},10^{-3},10^{-4}\}$. We use a cosine learning rate schedule, with peak learning rates of 0.01 for SGD, $10^{-4}$ for SignSGD, and $10^{-3}$ for Muon. Each configuration is trained under three multiple seeds, and we report their mean and sample standard deviation. Detailed curves for training loss, validation loss, and accuracy are presented in ~\cref{fig:sign_vs_sgd_llm,tab:optimizer-final-metrics}. We set momentum to zero for each optimizer.

\begin{figure}[tbp]
  \centering
  \includegraphics[width=\textwidth]{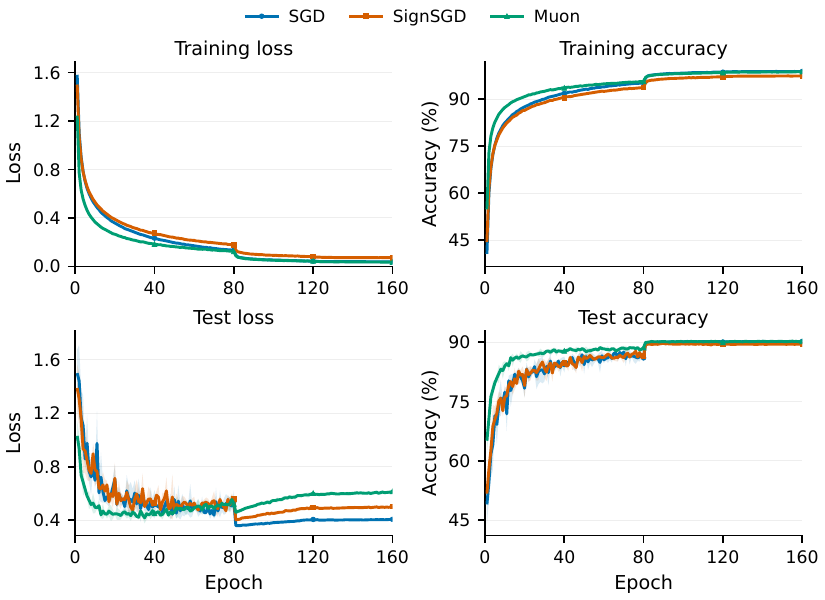}
  \caption{SGD, SignSGD, and Muon on ResNet-20/CIFAR-10. The top row shows training loss and training accuracy, and the bottom row shows test loss and test accuracy. Curves show means over three seeds, and shaded bands indicate $\pm1$ sample standard deviation.}
  \label{fig:sign_vs_sgd_cnn}
\end{figure}

\begin{figure}[!htb]
  \centering
  \includegraphics[width=\textwidth]{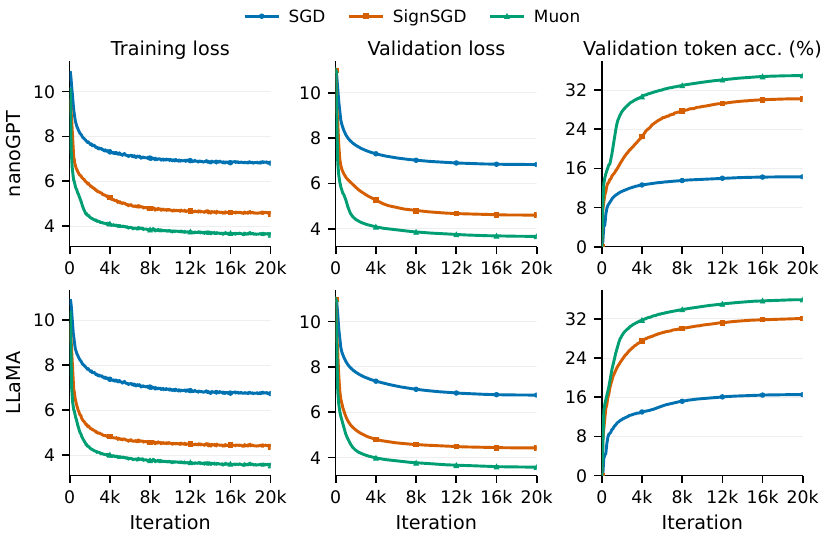}
  \caption{SGD, SignSGD, and Muon on C4. The top row is nanoGPT, and the bottom row is LLaMA. The columns show training loss, validation loss, and validation token accuracy. Curves show means over three runs.}
  \label{fig:sign_vs_sgd_llm}
\end{figure}

\begin{table}[!htb]
  \centering
  \caption{Means over three seeds, with sample standard deviations in small gray subscripts. Held-out loss and accuracy use the test set for ResNet-20 and the validation set for the language models.}
  \label{tab:optimizer-final-metrics}
  \begingroup
  \small
  \renewcommand{\arraystretch}{1.15}
  \setlength{\tabcolsep}{8pt}
  \newcommand{\metricspread}[1]{\raisebox{-0.4ex}{\fontsize{6.5}{7}\selectfont\textcolor{black!45}{#1}}}
  \begin{tabular}{@{}clr@{\hspace{0.10em}}lr@{\hspace{0.10em}}lr@{\hspace{0.10em}}l@{}}
    \toprule
    Model & Optimizer & \multicolumn{2}{c}{Training loss} & \multicolumn{2}{c}{\makecell{Test/validation\\loss}} & \multicolumn{2}{c}{Accuracy (\%)} \\
    \midrule
    & SGD & $0.0381$ & \metricspread{0.0018} & $0.407$ & \metricspread{0.006} & $90.19$ & \metricspread{0.24} \\
    ResNet-20 & SignSGD & $0.0716$ & \metricspread{0.0006} & $0.503$ & \metricspread{0.008} & $89.44$ & \metricspread{0.09} \\
    & Muon & $0.0350$ & \metricspread{0.0021} & $0.616$ & \metricspread{0.015} & $90.18$ & \metricspread{0.14} \\
    \midrule
    & SGD & $6.809$ & \metricspread{0.036} & $6.835$ & \metricspread{0.016} & $14.29$ & \metricspread{0.09} \\
    nanoGPT & SignSGD & $4.551$ & \metricspread{0.013} & $4.603$ & \metricspread{0.014} & $30.22$ & \metricspread{0.05} \\
    & Muon & $3.611$ & \metricspread{0.028} & $3.666$ & \metricspread{0.005} & $34.96$ & \metricspread{0.07} \\
    \midrule
    & SGD & $6.735$ & \metricspread{0.046} & $6.759$ & \metricspread{0.012} & $16.54$ & \metricspread{0.11} \\
    LLaMA & SignSGD & $4.375$ & \metricspread{0.014} & $4.428$ & \metricspread{0.006} & $32.06$ & \metricspread{0.10} \\
    & Muon & $3.528$ & \metricspread{0.028} & $3.583$ & \metricspread{0.005} & $35.92$ & \metricspread{0.07} \\
    \bottomrule
  \end{tabular}
  \endgroup
\end{table}

\begin{figure}[!htb]
  \centering
  \includegraphics[width=\linewidth]{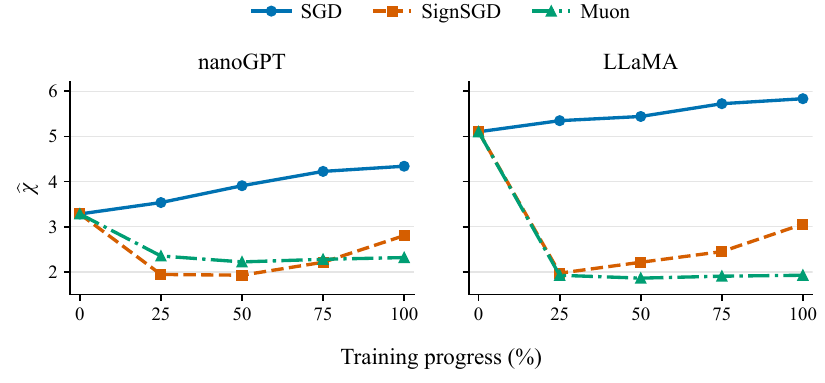}
  \caption{Coordinate-noise concentration $\widehat\chi$ during SGD, SignSGD, and Muon training on nanoGPT and LLaMA. Markers show measurements at five frozen checkpoints.}
  \label{fig:optimizer-coordinate-chi}
\end{figure}

\subsection{Direct Coordinate-Noise Measurement}
\label{app:coordinate-noise-measurement}
We measure coordinate noise at frozen checkpoints of one run per model--optimizer pair, using the training configurations in \cref{app:performance-configurations}. Along the training trajectory, we store the checkpoints at different stages, and then we use a large number of batches to estimate the true gradient and noise distribution. 

At a frozen checkpoint, let $g_i^{(k)}$ denote coordinate $i$ of the $k$th raw training mini-batch gradient and let $\overline g_i=K^{-1}\sum_{k=1}^K g_i^{(k)}$. The estimator is
\begin{equation}
\label{eq:coordinate-sample-variance}
\widehat\sigma_i^2=\frac{1}{K-1}\sum_{k=1}^K(g_i^{(k)}-\overline g_i)^2,
\qquad \widehat\sigma_i=\sqrt{\widehat\sigma_i^2}.
\end{equation}
Each checkpoint estimate uses $K=10000$ raw training mini-batch gradients.

We use the \emph{empirical} counterpart of the vector concentration coefficient defined in~\eqref{eq:chi-sigma},
\begin{equation}
\label{eq:empirical-coordinate-chi}
\widehat\chi
=\frac{d\sum_{i=1}^d\widehat\sigma_i^2}
       {(\sum_{i=1}^d\widehat\sigma_i)^2}.
\end{equation}

All trainable coordinates are included, and $\widehat\chi$ is computed directly from their estimated standard deviations. We present estimated  $\widehat\chi$ across different optimizers and models in \cref{fig:optimizer-coordinate-chi}. We observe that under SGD, nanoGPT and LLaMA exhibit larger $\widehat\chi$ than ResNet-20 at all five measured stages, with LLaMA consistently the highest. Also, within each language model, SGD yields larger $\widehat\chi$ than SignSGD and Muon at every measured checkpoint after initialization.

For the distribution plot presented in \cref{fig:sgd-sigma-distribution}, define $x_i=\widehat\sigma_i/\overline{\widehat\sigma}$, where $\overline{\widehat\sigma}=d^{-1}\sum_i\widehat\sigma_i$, and $z_i=\log_{10}x_i$. The curve is a smoothed density $\widehat p(z)$ whose area satisfies $\int_a^b\widehat p(z)\,dz\approx d^{-1}|\{i:a\le z_i\le b\}|$, the fraction of coordinates in the log interval $[a,b]$.

\newpage
\endgroup
\section*{NeurIPS Paper Checklist}

\begin{enumerate}

\item {\bf Claims}
    \item[] Question: Do the main claims made in the abstract and introduction accurately reflect the paper's contributions and scope?
    \item[] Answer: \answerYes{}
    \item[] Justification: We explicitly state our theoretical contributions and empirical findings regarding optimization under sparse noise in the abstract and introduction. These claims are strictly supported by the rigorous mathematical proofs and experimental evaluations presented in the subsequent sections.

\item {\bf Limitations}
    \item[] Question: Does the paper discuss the limitations of the work performed by the authors?
    \item[] Answer: \answerYes{}
    \item[] Justification: The theoretical results rely on the assumptions explicitly stated in Sections~3.1 and~5.1. Section~5.1 specifies that the Muon analysis considers pure polar updates without Newton--Schulz iterations or momentum.

\item {\bf Theory assumptions and proofs}
    \item[] Question: For each theoretical result, does the paper provide the full set of assumptions and a complete (and correct) proof?
    \item[] Answer: \answerYes{}
    \item[] Justification: The paper provides assumptions and proofs for each theoretical result.

    \item {\bf Experimental result reproducibility}
    \item[] Question: Does the paper fully disclose all the information needed to reproduce the main experimental results of the paper to the extent that it affects the main claims and/or conclusions of the paper (regardless of whether the code and data are provided or not)?
    \item[] Answer: \answerYes{}
    \item[] Justification: The paper discloses the information necessary to reproduce the main experimental results in Appendix~\ref{app:experimental_details}.

\item {\bf Open access to data and code}
    \item[] Question: Does the paper provide open access to the data and code, with sufficient instructions to faithfully reproduce the main experimental results, as described in supplemental material?
    \item[] Answer: \answerYes{}
    \item[] Justification: We provide the code and instructions in the supplementary material. Please follow the instructions in \texttt{README.md} to reproduce our results. Our code is also open-sourced at \url{https://github.com/Dingzhen230/SignSGD_Outperforms_SGD}.

\item {\bf Experimental setting/details}
    \item[] Question: Does the paper specify all the training and test details (e.g., data splits, hyperparameters, how they were chosen, type of optimizer) necessary to understand the results?
    \item[] Answer: \answerYes{}
    \item[] Justification: The setup of our experiments is clearly introduced in Appendix~\ref{app:experimental_details}.

\item {\bf Experiment statistical significance}
    \item[] Question: Does the paper report error bars suitably and correctly defined or other appropriate information about the statistical significance of the experiments?
    \item[] Answer: \answerYes{}
    \item[] Justification: Our real-data experiments report means and sample standard deviations over three runs in Appendix~\ref{app:performance-configurations}.

\item {\bf Experiments compute resources}
    \item[] Question: For each experiment, does the paper provide sufficient information on the computer resources (type of compute workers, memory, time of execution) needed to reproduce the experiments?
    \item[] Answer: \answerNo{}
    \item[] Justification: We report the hardware and training configurations in Appendix~\ref{app:performance-configurations} and release the experimental code. Execution-time records are unavailable, and per-experiment runtimes and total compute are not fully reported.

\item {\bf Code of ethics}
    \item[] Question: Does the research conducted in the paper conform, in every respect, with the NeurIPS Code of Ethics \url{https://neurips.cc/public/EthicsGuidelines}?
    \item[] Answer: \answerYes{}
    \item[] Justification: The research conducted in the paper conforms with the NeurIPS Code of Ethics.

\item {\bf Broader impacts}
    \item[] Question: Does the paper discuss both potential positive societal impacts and negative societal impacts of the work performed?
    \item[] Answer: \answerNA{}
    \item[] Justification: This is primarily a theoretical paper with no potential negative social impact.

\item {\bf Safeguards}
    \item[] Question: Does the paper describe safeguards that have been put in place for responsible release of data or models that have a high risk for misuse (e.g., pre-trained language models, image generators, or scraped datasets)?
    \item[] Answer: \answerNA{}
    \item[] Justification: This is primarily a theoretical paper with no such risks.

\item {\bf Licenses for existing assets}
    \item[] Question: Are the creators or original owners of assets (e.g., code, data, models), used in the paper, properly credited and are the license and terms of use explicitly mentioned and properly respected?
    \item[] Answer: \answerYes{}
    \item[] Justification: We cite the creators of the dataset used in our experiments, and we explicitly respect the license and terms of use.

\item {\bf New assets}
    \item[] Question: Are new assets introduced in the paper well documented and is the documentation provided alongside the assets?
    \item[] Answer: \answerNA{}
    \item[] Justification: This work does not release new assets.

\item {\bf Crowdsourcing and research with human subjects}
    \item[] Question: For crowdsourcing experiments and research with human subjects, does the paper include the full text of instructions given to participants and screenshots, if applicable, as well as details about compensation (if any)?
    \item[] Answer: \answerNA{}
    \item[] Justification: The paper does not involve crowdsourcing nor research with human subjects.

\item {\bf Institutional review board (IRB) approvals or equivalent for research with human subjects}
    \item[] Question: Does the paper describe potential risks incurred by study participants, whether such risks were disclosed to the subjects, and whether Institutional Review Board (IRB) approvals (or an equivalent approval/review based on the requirements of your country or institution) were obtained?
    \item[] Answer: \answerNA{}
    \item[] Justification: The paper does not involve crowdsourcing nor research with human subjects.

\item {\bf Declaration of LLM usage}
    \item[] Question: Does the paper describe the usage of LLMs if it is an important, original, or non-standard component of the core methods in this research? Note that if the LLM is used only for writing, editing, or formatting purposes and does \emph{not} impact the core methodology, scientific rigor, or originality of the research, declaration is not required.
    %this research?
    \item[] Answer: \answerNA{}
    \item[] Justification: The core method development in this research does not involve LLMs as any important, original, or non-standard components.

\end{enumerate}

\end{document}